\documentclass{article}

\PassOptionsToPackage{numbers, compress}{natbib}

\usepackage[final]{neurips_glfrontiers_2023}

\usepackage[utf8]{inputenc} 
\usepackage[T1]{fontenc}    
\usepackage[hidelinks]{hyperref}       
\usepackage{url}            
\usepackage{booktabs}       
\usepackage{wrapfig}        
\usepackage{multirow}
\usepackage{caption}
\usepackage{subcaption}     
\usepackage{nicefrac}       
\usepackage{microtype}      
\usepackage{xcolor}         
\usepackage[inline]{enumitem}

\def\sS{{\mathbb{S}}}
\def\sR{{\mathbb{R}}}
\def\sA{{\mathbb{A}}}

\def\hrepr{{\mathbf{H}}}

\def\cL{{\mathcal{L}}}
\def\cR{{\mathcal{R}}}
\def\cS{{\mathcal{S}}}
\def\cT{{\mathcal{T}}}
\def\cV{{\mathcal{V}}}

\DeclareMathAlphabet{\mathsfit}{\encodingdefault}{\sfdefault}{m}{sl}
\SetMathAlphabet{\mathsfit}{bold}{\encodingdefault}{\sfdefault}{bx}{n}
\newcommand{\tens}[1]{\bm{\mathsfit{#1}}}
\def\tA{{\tens{A}}}

\newcommand{\kgadj}{\tA}

\newcommand{\kgspace}{\sA_{N, R}}

\newcommand{\namekg}{attributed graph\xspace}
\newcommand{\namekgs}{attributed graphs\xspace}

\newcommand{\nameKGs}{Attributed Graphs\xspace}

\newcommand{\namedekgs}{double-exchangeable \namekgs}

\newcommand{\namemtdekg}{multi-task double-exchangeable \namekg}
\newcommand{\namemtdekgs}{multi-task double-exchangeable \namekgs}
\newcommand{\nameMtdekg}{Multi-task double-exchangeable \namekg}

\newcommand{\nameLAYER}{Multi-Task Double-Equivariant Linear Layer\xspace}

\newcommand{\namelayerabbr}{MTDE linear layer\xspace}

\newcommand{\namelayersabbr}{MTDE linear layers\xspace}

\newcommand{\nameMODEL}{Multi-Task Double-Equivariant Architecture\xspace}

\usepackage{amsmath,amsfonts,bm,dsfont,diagbox,amsthm,mathtools,tikz-cd}
\usepackage{cleveref}
\usepackage{xspace}
\theoremstyle{plain}
\newtheorem{theorem}{Theorem}[section]

\theoremstyle{definition}
\newtheorem{definition}[theorem]{Definition}

\usepackage{thm-restate}

\newcommand{\revision}[1]{{\color{black}#1}}
\newcommand{\revisioncamera}[1]{{\color{black}#1}}
\newcommand{\eat}[1]{}

\usepackage{multibib}
\newcites{APP}{Additional References in the Appendix}

\title{A Multi-Task Perspective for Link Prediction with New Relation Types and Nodes}

\author{%
    Jincheng Zhou \\
    Purdue University\\
    \texttt{zhou791@purdue.edu} \\
    \And
    Beatrice Bevilacqua \\
    Purdue University\\
    \texttt{bbevilac@purdue.edu} \\
    \And
    Bruno Ribeiro \\
    Purdue University \\
    \texttt{ribeiro@cs.purdue.edu} \\
}

\begin{document}

\maketitle

\begin{abstract}

The task of inductive link prediction in (discrete) attributed multigraphs infers missing attributed links (relations) between nodes in new test multigraphs. Traditional relational learning methods face the challenge of limited generalization to test multigraphs containing both novel nodes and novel relation types not seen in training. Recently, under the only assumption that all relation types share the same {\bf structural} predictive patterns (single task), Gao et al.\ (2023) proposed a link prediction method using the theoretical concept of {\em double equivariance}  (equivariance for nodes \& relation types), in contrast to the (single) equivariance (only for nodes) used to design Graph Neural Networks (GNNs). In this work we further extend the double equivariance concept to {\em multi-task double equivariance}, where we define link prediction in attributed multigraphs that can have distinct and potentially conflicting predictive patterns for different sets of relation types (multiple tasks). Our empirical results on real-world datasets demonstrate that our approach can effectively generalize to test graphs with multi-task structures without access to additional information.
\end{abstract}

\section{Introduction}
Discrete attributed multigraphs  \revision{(e.g., knowledge graphs, multilayer networks, heterogeneous networks, etc.)}, which we refer as \namekgs for simplicity, have been widely used for modeling relational data, which can also be expressed as a collection of triplets. 
Storing factual knowledge in \namekgs enables their application across a wide variety of tasks, encompassing complex question answering~\citep{fu2020survey,huang2022endowing} and logical reasoning~\citep{chen2020review}. 
Since relational data is often incomplete, predicting missing triplets, or, equivalently, predicting the existence of a relation of a certain type between a pair of nodes is an important task. 
However, conventional methods are generally limited to predicting missing links for relation types observed during training.
As a consequence, standard attributed link prediction methods are incapable of making predictions that involve completely new relation types over completely new nodes (or new graphs), which is arguably the most difficult and perhaps the most interesting link prediction task in \namekgs.

In this work we focus on the task of predicting missing triplets in test \namekgs that contain 
\revision{\textit{completely} new nodes and new relation types (i.e., no training nodes and no training relations).}
We assume that no extra information is available either at train or test time, apart from the input (observable) graph with its nodes and relation types. As a result, existing zero-shot methods~\citep{qin2020generative,geng2021ontozsl,li2022hierarchical}, which rely on textual descriptions
\revision{and/or ontological information}
of the relation types, and few-shot learning methods~\citep{xiong2018one,chen2019meta,sun2021one,zhang2020few,qian2022few}, 
\revision{which require a shared observable graph between train and test set,}
are unable to perform our task. 
\revision{Recently, \citet{kg-equivariance,lee2023ingram} have tackled this problem by proposing novel approaches capable of generalizing to completely new nodes and relation types without requiring any extra information. More precisely, in order to solve this task, \citet{kg-equivariance} introduced double (permutation) equivariant models for attributed graphs. Intuitively, these models treat every type of relation and all nodes as interchangeable with each other, effectively capturing what can be referred to as the double exchangeability assumption.}
Since the new test relation types are likewise assumed to be exchangeable with the training ones, the missing test triplets can be predicted directly using knowledge acquired during training. However, in some real-world \namekgs, this double exchangeability may leave performance on the table since it can make the model excessively equivariant.
For instance, consider the scenario in \Cref{fig:intro}  where the training graph comprises two weakly-connected knowledge bases representing different sports communities (racing and gymnastic) with some different relation types. The predictive patterns for relation types in those communities might differ and potentially be contradictory, implying that \emph{not all} relation types are exchangeable. In our example, two racing teammates necessarily have the same team principal in the racing community, and therefore (Horner, \texttt{team\_principal\_of}, Pérez) can be predicted from ((Horner, \texttt{team\_principal\_of}, Verstappen), (Verstappen, \texttt{teammate\_of}, Pérez)). On the contrary, two gymnastic teammates might have different coaches and (Landi, \texttt{coach\_of}, Lee) should not be predicted when seeing ((Landi, \texttt{coach\_of}, Biles), (Biles, \texttt{teammate\_of}, Lee)). Due to this contradictory patterns, \texttt{team\_principal\_of} and \texttt{coach\_of} are not exchangeable. Similarly, the new test relation types might be exchangeable only with a subset of the training ones. In our example, \texttt{boss\_of} is exchangeable with \texttt{team\_principal\_of}, but not with \texttt{coach\_of}.

\begin{figure}[t]
    \centering
    \includegraphics[width=\textwidth]{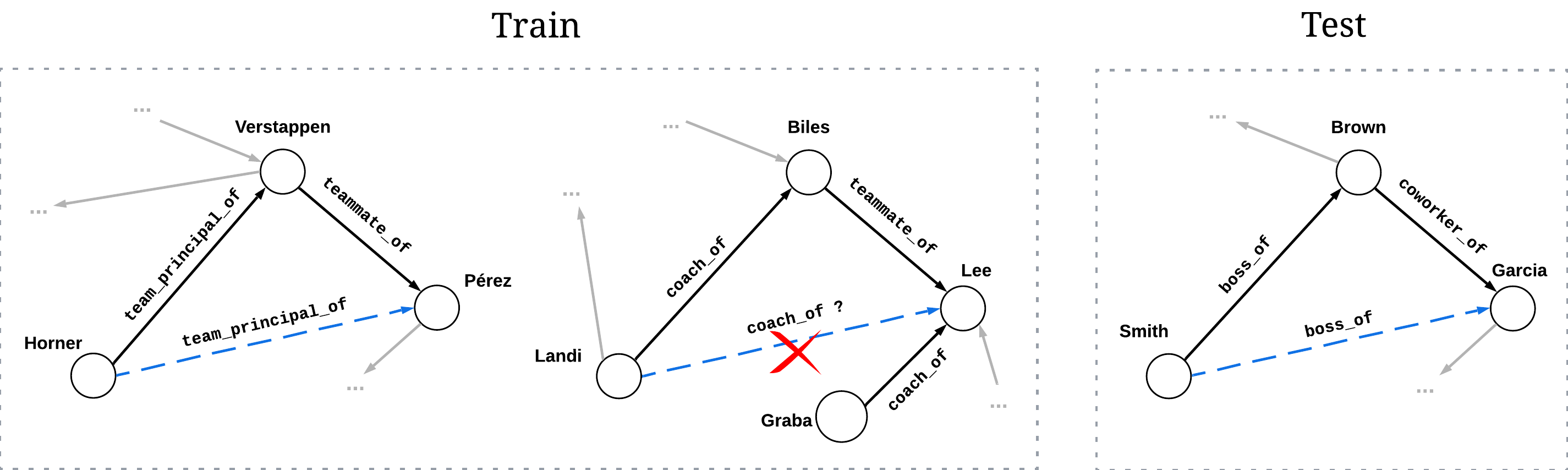}
    \caption{An example of a setting where not all relation types are exchangeable. Train graph contains relation types in the racing and the gymnastic communities. Test graph contains relation types in a business community. } 
    \label{fig:intro}
    \vspace{-10pt}
\end{figure}

\paragraph{Our approach.} Our work relaxes the double equivariance proposed by~\citet{kg-equivariance} by learning to partition the set of relations into distinct clusters, where each cluster exclusively contains relation types that are exchangeable among themselves. 
We demonstrate that these clusters of relation types can be understood as distinct tasks in a multi-task setting.
Consequently, our method learns multiple double equivariant graph models, one for each task (cluster).
At test time, we employ an adaptation procedure to assign new relation types to the most appropriate cluster, thus ensuring generalization to previously unseen relation types.

\paragraph{Main contributions.}
Our main contributions are as follows:
\begin{enumerate*}
    \item We develop a method capable of modeling the existence of distinct and contradictory predictive patterns among various sets of relation types by treating them as separate tasks in a multi-task setting;
    \item We propose a test-time adaptation procedure that learns task assignments for new relation types, enabling the application of our proposed method to entirely new test relation types;
    \item We create new benchmark datasets that fit the multi-task scenario we focus on;
    \item We develop a novel evaluation metric to more effectively measure the performance of existing methods in predicting missing triplets.
\end{enumerate*}

\section{Related work}
\paragraph{Link prediction in \namekgs.} Existing link prediction methods in \namekgs can be categorized into factorization-based approaches~\citep{nickel2011three,bordes2013translating,wang2014knowledge,yang2015embedding,nickel2016holographic,trouillon2016complex,lacroix2018canonical,dettmers2018convolutional,nguyen2018novel,sun2018rotate,Chen2021RelationPA} and GNN-based models~\citep{schlichtkrull2018modeling,vashishth2020somposition,galkin2020message,yu2020generalized,zhang2022rethinking}. Although the former exhibit remarkable performances, especially when combined with appropriate training strategies \citep{ruffinelli2020You,jain2020knowledge}, they are typically restricted to transductive settings. Conversely, GNN-based models can also be applied to inductive scenarios involving new nodes in test~\citep{teru2020inductive,zhu2021neural,ali2021improving,zhang2022knowledge,galkin2022nodepiece}. All these methods, however, cannot work when presented with new relation types in test, which instead represents the main interest of our paper.

\paragraph{Zero-shot and few-shot learning on new relation types (with side information).}
Recent works, aiming to predict missing triplets involving new relation types, consider the zero-shot or few-shot learning paradigm. To generalize to new relation types, zero-shot methods~\citep{qin2020generative,geng2021ontozsl,li2022hierarchical} use additional contextual information, such as the semantic descriptions of the relation types, making them unfit for our scenario. 
\revision{
In contrast, few-shot methods primarily adopt a meta-learning paradigm~\citep{xiong2018one,chen2019meta,sun2021one,zhang2020few,qian2022few} to discover similarities between the new relation types and the ones used during training, by learning from a limited number of support triplets. These methods, however, typically require the support triplets to be connected to the seen nodes and relation types. For instance, \citet{qian2022few} proposed to measure the similarities between the subgraphs around the target triplet involving the unseen relations with those involving the training relations, assuming the presence of a shared observable graph between training and test sets. 
This setting is more constrained than ours since we consider a completely new graph at test time involving no nodes and relations seen during training.
}

\paragraph{Double inductive link prediction (without side information).}
\revision{
Our work extends  ISDEA~\citep{kg-equivariance}, which introduces the double inductive link prediction task and the concept of double (permutation) equivariant graph models. These models are specifically designed to capture the double exchangeability assumption on the data, where, intuitively, nodes as well as relation types are interchangeable with each other. Importantly, these models are capable of generalizing to unseen relations, aligning to our task of interest. ISDEA also
demonstrated that InGram~\citep{lee2023ingram}, another recent approach for handling new relation types in test, generates positional embeddings that exhibit double equivariance in distribution, and introduced a corresponding method to enhance InGram's model, referred to as DEq-InGram.
In this paper we build upon ISDEA's model and address the shortcomings of double exchangeability. 
As we shall see next, approaches modeling the double-exchangeability assumption are unable} to properly model the difference in predictions between non-exchangeable relation types that have different (and potentially contradictory) predictive patterns. Due to space constraints, we refer the reader to \Cref{app:related-work} for detailed comparisons with prior work.

\section{Problem Definition} \label{sec:problem-def}

In this section we introduce the notation used through the remainder of this work. We consider an \namekg (multigraph) as a finite collection of typed relations between nodes. Formally, let $\cV$ be a finite discrete set of nodes and $\cR$ a finite discrete set of relation types. A triplet $(u, r, v)$ in the \namekg indicates that a node $u \in \cV$ is linked to another node $v \in \cV$ by means of a relation of type $r \in \cR$. Without loss of generality, we assume node and relation sets are numbered, that is $\cV \coloneqq \{ 1, 2, \dots, N \}$ and $\cR \coloneqq \{ 1, 2, \dots, R \}$, where $N \geq 2$ and $R \geq 2$. Consequently, 
we can represent an \namekg in tensor form as $\kgadj \in \kgspace$, with $\kgspace = \{ 0, 1 \}^{N \times R \times N}$, where $\kgadj_{u, r, v} = 1$ if and only if $(u, r, v)$ is a triplet in the \namekg.

Link prediction in \namekgs can be cast as a self-supervised learning problem~\citep[Appendix B]{cotta2023causal}, where an input graph $\kgadj$ is assumed to be the result of the application of an unknown mask $M \in   \{ 0, 1 \}^{N \times R \times N} $ on an unknown \namekg $\kgadj^{(\text{full})}$, i.e., $\kgadj = M \odot \kgadj^{(\text{full})}$ with $\odot$ the element-wise product, where the masking process hides the existence of certain triplets. The goal of a model is to predict the existence of the masked (or missing) triplets from $\kgadj$. That is, if $\overline{M}$ is the complement mask defined as $\overline{M} = {\bf 1} - M$, the model is asked to predict $P(\overline{M} \odot \kgadj^{(\text{full})} \mid \kgadj)$. 

\emph{In this work we focus on predicting missing triplets in new \namekgs with new relation types}. Given a training graph $\kgadj^{\text{(tr)}}$, with node set $\cV^{\text{(tr)}}$ and relation set $\cR^{\text{(tr)}}$, we aim to learn a model capable of accurately predicting missing triplets in a test graph $\kgadj^{\text{(te)}}$, with node set $\cV^{\text{(te)}}$ and relation set $\cR^{\text{(te)}}$, involving both new nodes and new relations types: $\cV^{\text{(tr)}} \not \supseteq \cV^{\text{(te)}}$ and $\cR^{\text{(tr)}}  \not \supseteq \cR^{\text{(te)}}$. To accurately predict missing test triplets without any extra information, such as contextual information (as in zero-shot methods~\citep{qin2020generative,geng2021ontozsl,li2022hierarchical}) or task labels (few-shot methods~\citep{xiong2018one,chen2019meta,sun2021one,zhang2020few,qian2022few}), $\kgadj^{\text{(te)}}$ must exhibit predictive patterns found in $\kgadj^{\text{(tr)}}$, which implies that missing test triplets can be predicted using the knowledge acquired from training, even if the relations convey entirely different meanings. 
\revision{\citet{kg-equivariance} recently proposed double equivariant models, capturing the concept of double exchangeability, that can solve our task of interest, namely the existence of new relation types at test time without extra information}. However, as we discuss next, the methods  in \citet{kg-equivariance} require that all relation types share a single predictive pattern.

\begin{figure}[t]
    \centering
    \includegraphics[width=1.0\textwidth]{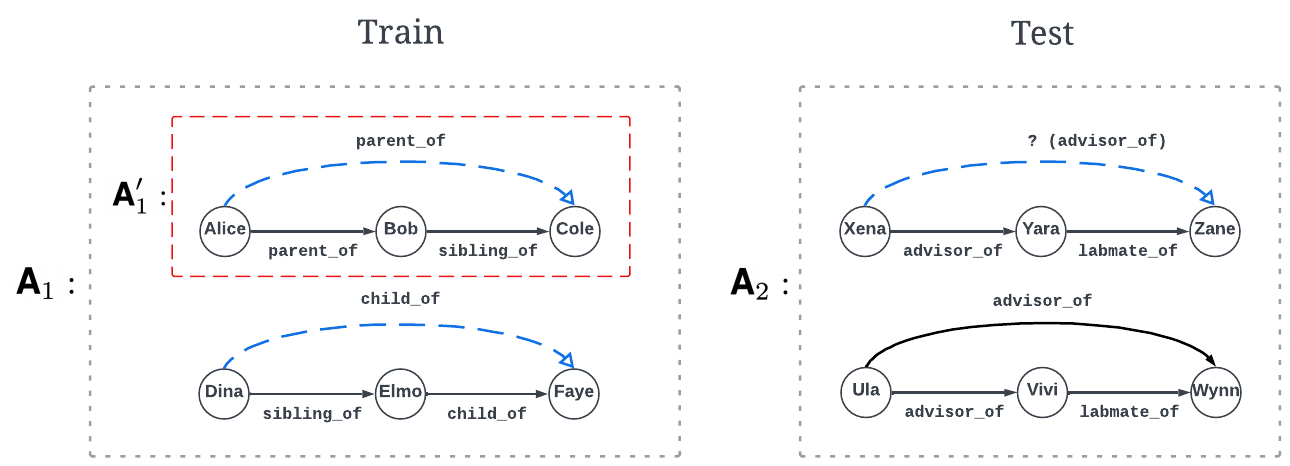}
    \caption{Example of our multi-task scenario.
    Solid arrows constitute the input (or observable) triplets, whereas the blue, dashed arrows represent the missing triplets to be predicted.
    The training \namekg exhibits conflicting predictive patterns for \texttt{parent\_of} and \texttt{child\_of}. The new test relation \texttt{advisor\_of} follows the same predictive pattern as \texttt{parent\_of}. } 
    \label{fig:multi-task}
    \vspace{-8pt}
\end{figure}

\paragraph{Existing gap: Learning to predict new relation types on graphs with conflicting predictive patterns.} To the best of our knowledge, no existing method \revision{aims at solving the problem of predicting} missing triplets involving new relation types without the need of extra information \emph{when relations exhibit different and potentially conflicting predictive patterns}. 
\Cref{fig:multi-task} illustrates an example of this setting, with solid arrows representing the input graph and dashed arrows denoting the missing triplets to be predicted. In this example, the training graph $\kgadj_1$ has family-tree relationships, while the test graph $\kgadj_2$ has relations in academia. Our goal is to learn predictive patterns from $\kgadj_1$ and generalize to predict the missing triplets in $\kgadj_2$. 
However, as demonstrated in the figure, the training graph $\kgadj_1$ contains conflicting predictive patterns.
The type of relation between Alice and Cole is the \emph{first} relation type in the 2-hop chain ((Alice, \texttt{parent\_of}, Bob), (Bob, \texttt{sibling\_of}, Cole)), which is \texttt{parent\_of}. Conversely, the type of relation between Dina and Faye is the \emph{second} relation type in the 2-hop chain ((Dina, \texttt{sibling\_of}, Elmo), (Elmo, \texttt{child\_of}, Faye)), which is \texttt{child\_of}.
Having conflicting predictive patterns in training does not prevent accurate prediction of test triplets, as long as we can correctly identify which pattern, among the learned ones, a test relation follows. In our example,
the true test relation type \texttt{advisor\_of} between Xena and Zane in $\kgadj_2$ shares the same predictive pattern as the relation \texttt{parent\_of} in $\kgadj_1$, as can be inferred from the observable triplet (Ula, \texttt{advisor\_of}, Wynn).
Since \revision{the pioneering work by \citet{kg-equivariance} focuses on a single double equivariant model, it learns a single model for all predictive patterns, which makes the model unable to distinguish the predictions of \texttt{parent\_of} and \texttt{child\_of}, and, consequently, between those of \texttt{advisor\_of} and \texttt{labmate\_of} in the example above}.

\section{A Multi-Task Perspective on Inductive Learning of New Relation Types}
In the previous section we noted that accurate prediction of missing triplets in a test graph with completely new relation types requires the test graph to exhibit the \revision{similar} predictive patterns as the training graph. Expanding upon this concept, we next define a predictive pattern through the notion of \revision{relation type} exchangeability and describe a task as a set of relation types sharing the same predictive patterns. This allows us to frame the problem of predicting with new relation types as \revision{entailing multiple tasks}, where \revision{we learn different predictive patterns for each task}.
Correctly performing on all tasks will then result in accurate predictions of test triplets having new relation types, as long as we can identify to which task they belong to.

\subsection{Re-imagining Inductive Learning as Relational Tasks}
We begin by defining the concept introduced in \citet{kg-equivariance} of exchangeability between relation types, which can informally be understood as the property of (certain) relation types to be interchangeable with each other. This property encapsulates our notion of shared predictive patterns, as exchangeable relation types necessarily follow the same patterns.

\begin{definition}[Exchangeability between relation types \citep{kg-equivariance}] \label{def:ex-relation}
  \revision{Let $\kgadj \in \kgspace$ be a random variable representing an \namekg}
  with node set $\cV = \{ 1, 2, \dots, N \}$ and relation set $\cR = \{ 1, 2, \dots, R \}$ and two relation types $r, r' \in \cR$, we say that $r$ and $r'$ are \textit{exchangeable} if there exists some node permutation $\pi \in \sS_N$ and relation type permutation $\sigma \in \sS_R$ such that the following two conditions are satisfied:
    \begin{align*}
        \sigma \circ r = r'  \qquad \text{and} \qquad
        P(\kgadj) = P(\sigma \circ \pi \circ \kgadj) ,
    \end{align*}
    where 
    \revisioncamera{$\circ$ denotes the permutation actions of $\pi$ on the nodes and $\sigma$ on the relation types in a graph.}
    That is, for all $u, v \in \cV$ and $r \in \cR$, the symmetric group $\sS_{N}$ and the symmetric group $\sS_{R}$ act on a graph via $(\pi \circ \kgadj)_{\pi \circ u, r, \pi \circ v} = \kgadj_{u, r, v}$  and $(\sigma \circ \kgadj)_{u, \sigma \circ r, v} = \kgadj_{u, r, v}$.
    We denote the exchangeability between relation types by $\sim_e$. 
\end{definition}

A similar formalization can be made for nodes \revision{(as commonly defined in the graph neural network literature~\citep{xu2018powerful,morris2019weisfeiler})}. Due to space limits, we omit this definition, but we consider nodes to be exchangeable throughout the paper.

In the following, we introduce our theoretical contributions. We start by proving that the exchangeability property between relation types can be regarded as a higher-order relation between the elements in $\cR$, or, more precisely, as an equivalence relation on $\cR$. 

\begin{restatable}{lemmma}{equivalence}
\label{lemma:equivalence}
    The exchangeability between relation types $\sim_e$ defines an \textbf{equivalence relation} on $\cR$, since it satisfies the reflexivity, symmetry, and transitivity properties.
\end{restatable}

The importance of \Cref{lemma:equivalence} is in that it allows us to partition the set of relations $\cR$ into disjoint equivalence classes. Each equivalence class contains relation types that are exchangeable with each other, whereas relation types that are not exchangeable belong to different equivalence classes. Consequently, these partitions of $\cR$ can naturally be considered as different tasks, dubbed \revision{{\em relational tasks} henceforth}, each containing relation types that share the same predictive patterns.

\begin{definition}[Relational tasks]\label{def:tasks}
    We define a \textit{relational task} with respect to a relation type $r \in \cR$ of an \namekg $\kgadj$ with node set $\cV$ and relation set $\cR$ to be the equivalence class $[r]$ under $\sim_e$, i.e.,
    \begin{align}\label{eq:tasks}
        \cT_r \coloneqq [r] = \{ r' \in \cR : r' \sim_e r \}.
    \end{align}
\end{definition}

Viewing $\cR $ as partitioned into a hidden set of disjoint relational tasks allows us to consider the problem from a multi-task perspective. Our goal then becomes finding a model able to accurately learn all tasks, specializing on the patterns unique to each task, which are potentially conflicting among each other. Such model will then be asked to recognize which task a test relation type belongs to, in order to predict missing test triplets by applying what was learned in train for the same task.

\subsection{Handling Conflicting Patterns for Different Relation Types as a Multi-task Scenario}
\Cref{def:ex-relation,def:tasks} allow us to formalize in the following definition the \namekgs of our interest, such as the one introduced in \Cref{sec:problem-def}, which we dub \namemtdekgs.
We adopt the term \emph{double exchangeable} from \citet{kg-equivariance}, since it inherently captures the idea of exchangeability both between node ids and between relation types -- a shared concept between our work and theirs -- but we extend it to our multi-task scenario.

\begin{definition}[\nameMtdekg] \label{def:double-ex}
    Given an \namekg $\kgadj$, with node set $\cV$ and relation set $\cR $, $\kgadj$ is said to be a \textit{\namemtdekg} if it has more than one relational tasks, i.e.,
    \begin{align*}
        |\{ \cT_r : r \in \cR \}| > 1.
    \end{align*}
    Equivalently, $\kgadj$ is a \namemtdekg if there exist two relation types $r, r' \in \cR$ such that $r \not \sim_e r'$.
\end{definition}

\revision{
In words, the training graph is a \namemtdekg, in which each task comprises exchangeable relation types that follow the same predictive patterns, while distinct tasks may have conflicting patterns. Furthermore, our test graph is a \namemtdekg, with tasks that constitute a subset of the training ones. 
}

\eat{
\subsection{The Special Case of Single Task Double-Exchangeable \nameKGs}
\Cref{def:double-ex} can be specialized into what we refer to as single-task \namedekgs if all relations belong to the same equivalence class (\Cref{lemma:equivalence}).
An example of this scenario happens when considering only the boxed subgraph $\kgadj_1'$ of the training graph in \Cref{fig:multi-task}. If we restrict our training graph to $\kgadj_1'$ and maintain $\kgadj_2$ as the test graph, then the true test relation type \texttt{advisor\_of} between Xena and Zane in $\kgadj_2$ can accurately be predicted by the model from \citet{kg-equivariance}. This is because it follows the \emph{only} predictive pattern present in the data, which is the one of the relation \texttt{parent\_of}. Nonetheless, as emphasized throughout our work, the single-task configuration represents a particular case of the more general multi-task setting, which accommodates a greater variety of \namekgs, such as the complete training graph $\kgadj_1$ in \Cref{fig:multi-task}.
}

\Cref{def:double-ex} can be specialized to the case of a single relational task (\Cref{def:tasks}), where all relations belong to the same equivalence class (\Cref{lemma:equivalence}).
An example of this scenario happens when considering only the boxed subgraph $\kgadj_1'$ of the training graph in \Cref{fig:multi-task}. If we restrict our training graph to $\kgadj_1'$ and maintain $\kgadj_2$ as the test graph, then the true test relation type \texttt{advisor\_of} between Xena and Zane in $\kgadj_2$ can accurately be predicted by the model from \citet{kg-equivariance}. This is because it follows the \emph{only} predictive pattern present in the data, which is the one of the relation \texttt{parent\_of}. Nonetheless, 
the single-task configuration represents a particular case of the more general multi-task setting, which 
\revision{accommodates the existence of potentially conflicting predictive patterns of different relations, such as those in} the complete training graph $\kgadj_1$ in \Cref{fig:multi-task}.

\section{Proposed Method} \label{sec:method}
In this section we introduce our framework to learn the different predictive patterns which are specific for each task and a procedure to generalize to new test relation types. Our proposed architecture models exchangeability between relation types belonging to the same task while differentiating them from the learned patterns of other tasks. To adapt to the unseen relations in the test \namekg, we propose a test-time adaptation procedure to identify the tasks to which test relations belong.

\subsection{Multi-Task Double-Equivariant Linear Layer} \label{sec:multi-task-de}

Suppose we knew that the ground-truth relational tasks $\{ \cT_r: r \in \cR \}$ in \Cref{eq:tasks} given an \namekg $\kgadj$ with node set $\cV$ and relation set $\cR$. \footnote{In \Cref{sec:attn-weights} we will remove this assumption and show how to learn task memberships.}
Without loss of generality, consider an arbitrary ordering of the relational tasks and denote the ordered relational tasks
as $\cT^{(1)}, \cT^{(2)}, \dots, \cT^{(K)}$, where 
\revisioncamera{$\cT^{(k)}$ is the $k$-th ordered task of a total number of $K$ tasks.}
We denote by $i: \cR \to \{ 1, 2, \dots, K \}$ a task index mapping, such that $\cT^{(i(r))} = \cT_r$. 
Inspired by the equivariant framework proposed by \citet{maron2020learning,bevilacqua2022equivariant} we present the following \textbf{\nameLAYER (\namelayerabbr)}, which updates representations at every layer $t$ as
\begin{align}  \label{eq:cond-eqvlin} 
        \hrepr^{(t+1)}_{\cdot, r, \cdot, \cdot} 
        &= L_1^{(t)} \big( \hrepr^{(t)}_{\cdot, r, \cdot, \cdot} \big)
        + L_2^{(t)} \big( \mathbf{1} \, \otimes \, p_{i(r)} + \sum_{r' \in \cT^{(i(r))} \setminus \{r\}} \hrepr^{(t)}_{\cdot, r', \cdot, \cdot} \big) \nonumber\\
        &+ \sum_{\substack{k = 1, \dots, K \\ k \neq i(r)}} L_3^{(t)} \big( \mathbf{1} \, \otimes \, p_{k} + \sum_{r'' \in \cT^{(k)}} \hrepr^{(t)}_{\cdot, r'', \cdot, \cdot} \big),
\end{align}
where $\hrepr^{(t)} \in \sR^{N \times R \times N \times d}$ is the layer input with $\hrepr^{(0)} = \kgadj$, 
\revisioncamera{$\hrepr^{(t)}_{\cdot, r, \cdot, \cdot}$ denotes the graph representations specific to relation type $r$,}
and $L_1^{(t)}, L_2^{(t)}, L_3^{(t)}: \sR^{N \times N \times d} \to \sR^{N \times N \times d'}$ are GNN layers that output \textit{pairwise} representations with $N$ the number of nodes, $R$ the number of relations and $d, d'$ appropriate dimensions. The vectors $p_k \in \sR^d$, $k = 1, 2, \dots, K$ are learnable positional embeddings, each specific to the task $\cT^{(k)}$, which are repeated on the last dimension through the Kronecker product with the matrix of all ones, $\mathbf{1} \in \{1\}^{N \times N \times 1}$. The sums in $\sum_{r' \in \cT^{(i(r))} \setminus \{r\}} \hrepr^{(t)}_{\cdot, r', \cdot, \cdot}$ and $\sum_{r'' \in \cT^{(k)}} \hrepr^{(t)}_{\cdot, r'', \cdot, \cdot}$ can be replaced by any other set aggregations.

Note that if the total number of relational tasks $K$ is 1, then \Cref{eq:cond-eqvlin} recovers the single-task double-equivariant layer proposed in \citet{kg-equivariance}.
Indeed, if $K=1$, then $i(r) = 1$ for any $r \in \cR$ with $\cT^{(1)} = \cR$, and \Cref{eq:cond-eqvlin} can be rewritten as 
\begin{equation} \label{eq:eqvlin}
    \hrepr^{(t+1)}_{\cdot, r, \cdot, \cdot}  = L^{(t)}_1(\hrepr^{(t)}_{\cdot, r, \cdot, \cdot}) + L^{(t)}_2\Bigl( \sum_{r' \in \cR \setminus \{r\}} \hrepr^{(t)}_{\cdot, r', \cdot, \cdot}  \Bigr),
\end{equation}
where the term $\mathbf{1} \, \otimes \, p_{i(r)}$ was absorbed into $L^{(t)}_2$.

\paragraph{The role of positional embeddings.} \Cref{eq:cond-eqvlin} uses the positional embedding vectors $p_{j} \in \sR^{d}$, $j \in \{1, \ldots, K\}$, to allow representations of relation types belonging to different tasks to be different, even when they have isomorphic observable graphs. In order to understand the role of $p_{j}$ in our architecture, we refer once again to \Cref{fig:multi-task}. Without the inclusion of the positional embeddings, \Cref{eq:cond-eqvlin} would give the same representation to the missing triplet involving $\texttt{parent\_of}$ and the missing triplet involving $\texttt{child\_of}$, even if those relations belong to two different relational tasks, because the inputs to $L^{(t)}_1,L^{(t)}_2,L^{(t)}_3$ are the same, starting from $t=0$.

\subsection{Learning Soft Task Membership via Attention Weights} \label{sec:attn-weights}

The previous section assumes we know the ground-truth assignment of relation types to tasks. In what follows, we learn such assignments from data only. 
Intuitively, we need to partition all relations $\cR$ into disjoint equivalence classes, where each partition corresponds to a unique relational task. 
This process is a discrete optimization problem, which we relax into a continuous one by means of an 
\revisioncamera{learnable}
attention matrix $\alpha \in [0, 1]^{R \times \hat{K}}$, where $\hat{K}$ is 
\revisioncamera{a hyperparameter controlling}
the maximum number of partitions we allow our architecture to model 
\revisioncamera{(which is potentially different from $K$, the ground-truth number of relational tasks unknown to us).}
The individual attention value $\alpha_{r, k}$ denotes the degree (or probability) that the relation $r \in \cR$ belongs to the $k$-th equivalence class, with the constraint that $\sum_{k=1}^\revision{\hat{K}} \alpha_{r, k} = 1$ for all $r \in \cR$. Hence, the \namelayerabbr of \Cref{eq:cond-eqvlin} can be relaxed into what we called the \textbf{soft \namelayerabbr}:
\begin{align}  \label{eq:asymcond-eqvlin}
    \hrepr^{(t+1)}_{\cdot, r, \cdot, \cdot} 
    &= L_1^{(t)} \big( \hrepr^{(t)}_{\cdot, r, \cdot, \cdot} \big)
    + L_2^{(t)} \big( \mathbf{1} \, \otimes \, p_{\hat{i}(r)} + \sum_{r' \in \cR \setminus \{r\}} \alpha_{r', \hat{i}(r)} \hrepr^{(t)}_{\cdot, r', \cdot, \cdot} \big) \nonumber\\
    &+ \sum_{\substack{k = 1, \dots, \revision{\hat{K}} \\ k \neq \hat{i}(r)}} L_3^{(t)} \big( \mathbf{1} \, \otimes \, p_{k} + \sum_{r'' \in \cR \setminus \{ r \}} \alpha_{r'', k} \hrepr^{(t)}_{\cdot, r'', \cdot, \cdot} \big),
\end{align}
where $\hat{i}(r) = \arg\max_{k= 1 \dots \revision{\hat{K}}} \alpha_{r, k}$, which ideally should give the correct id $i(r)$ of the ground-truth relational task $\cT_r$ that the relation $r$ belongs to.
The final architecture, which we name the {\bf \nameMODEL (MTDEA)}, is obtained by stacking $T$ soft \namelayersabbr to produce a graph representation $\Gamma(\kgadj) \in \sR^{N \times R \times N \times d}$ for a given \namekg $\kgadj$:
\begin{align*}
    \Gamma (\kgadj) \coloneqq L^{(T)} \big( f \big( \cdots f\big( L^{(1)}(\kgadj) \big) \cdots \big) \big),
\end{align*}
where $f$ is a non-polynomial activation such as ReLU. The predictions of individual triplets can then be obtained through a triplet score function $\Gamma_{\text{tri}}: \cV \times \cR \times \cV \times \kgspace \to [0, 1]$ followed by a sigmoid activation function, i.e., 
$
    \Gamma_{\text{tri}}((u, r, v), \kgadj) \coloneqq \sigma(\Gamma (\kgadj)_{u, r, v, \cdot}) .
$

\subsection{Dual-Sampling Loss with Task Membership Regularization}  \label{sec:loss}
Existing literature that tackles link prediction in \namekgs relies on a loss as based on entity-centric negative sampling \cite{yang2015embedding,schlichtkrull2018modeling,zhu2021neural}, where for each ground-truth (existing) triplet $(u, r, v)$, the tail node $v$ of $(u, r, v)$ is randomly corrupted to obtain a fixed number of negative samples $(u, r, v')$. Such entity-based negative sampling is insufficient for our loss because correctly predicting the \emph{relation type} between two nodes is equally important as correctly predicting the \emph{tail node} given the head node and relation type. To this end, we propose the \textit{dual-sampling task loss} $\cL_{\text{dual}}$, which given the training \namekg $\kgadj^{\text{(tr)}}$ with node set $ \cV^{\text{(tr)}}$ and relation set $ \cR^{\text{(tr)}}$, makes use of $n$ negative samples obtained by corrupting tail nodes and $m$ negative samples obtained by corrupting the relation types from positive samples, that is 
\begin{align}  \label{eq:task-loss}
    \cL_{\text{dual}} \coloneqq - \sum_{(u, r, v) \in \cS} \bigg( 
    &\log(\Gamma_{\text{tri}}((u, r, v), \kgadj^{\text{(tr)}})) 
    \revisioncamera{+} \frac{1}{n} \sum_{i=1}^n \log(1 - \Gamma_{\text{tri}}((u, r, v'_i), \kgadj^{\text{(tr)}}))) \nonumber \\
    &\revisioncamera{+} \frac{1}{m} \sum_{j=1}^m \log(1 - \Gamma_{\text{tri}}((u, r'_j, v), \kgadj^{\text{(tr)}})))
    \bigg) ,
\end{align}
where $\cS \coloneqq \{ (u, r, v) \in \cV^{\text{(tr)}} \times \cR^{\text{(tr)}} \times \cV^{\text{(tr)}} \mid \kgadj^{\text{(tr)}}_{u, r, v} = 1 \}$ is the set of positive triplets, $(u, r, v'_i)$ is the $i$-th entity-based negative sample and $(u, r'_j, v)$ the $j$-th relation-based negative sample corresponding to the positive triplet $(u, r, v)$.

\Cref{eq:task-loss} constitutes only a term of the loss function we optimize, which further contains regularization terms on the attention matrix $\alpha$.
Intuitively, we want the individual attention values to be either 0 or 1, because each value should represent whether a relation type belongs to certain task (value 1) or not (value 0). Moreover, we aim to have a large concentration of the attention values, in order to have as few partitions as possible. Hence, we propose the following model loss, where $\lambda_1, \lambda_2 \in \mathbb{R}$ are hyper-parameters weighting the terms:
\begin{align} \label{eq:full-loss}
    \mathcal{L} = \mathcal{L}_{\text{dual}} + \lambda_1 \underbrace{\sum_{r \in \cR^{\text{(tr)}}} \!\!\! \Big( -\sum_{j= 1 \dots \revision{\hat{K}}} \alpha_{r, j} \log \alpha_{r, j}\Big)}_{\mathcal{L}_\text{1-hot}} + \lambda_2 \underbrace{\Big(- \sum_{j= 1 \dots \revision{\hat{K}}} \text{LGamma} \Big(1 + \sum_{r \in \cR^{\text{(tr)}}} \alpha_{r,j} \Big)\Big)}_{\mathcal{L}_\text{conc}}. 
\end{align}
The first term $\mathcal{L}_{\text{dual}}$ is the dual-sampling loss in \Cref{eq:task-loss}. 
\revisioncamera{The second term $\mathcal{L}_\text{1-hot}$ \textit{minimizes} the entropy of $\alpha_{r, \cdot}$, i.e. the relation type $r$'s partition membership probabilities, for each relation type $r$. This effectively pushes $\alpha_{r, \cdot}$ towards a one-hot vector, encouraging individual attention values to be close to either 0 or 1.}
Finally, the third term $\mathcal{L}_\text{conc}$ takes advantage of the log-gamma function to encourage the relation set to be split in as few partitions as possible.

\subsection{A Test-Time Adaptation Procedure} \label{sec:transductive}
The attention matrix learned during training encodes task membership of relation types in training, and therefore it cannot be directly ported to the new test relation types in our test graph $\kgadj^{\text{(te)}}$ with $N^{\text{(te)}}$ nodes and $R^{\text{(te)}}$ relation types. We address this issue by adopting a test-time adaptation procedure where we optimize a
\revisioncamera{new}
test-time attention matrix $\alpha^{(\text{te})} \in [0, 1]^{R^{\text{(te)}} \times \hat{K}}$, 
\revisioncamera{$\alpha^{(\text{te})} \neq \alpha^{(\text{tr})}$,}
while freezing all other parameters of the architecture. During the adaptation, only the observable triplets of the test graph $\kgadj^{\text{(te)}}$ are used for training $\alpha^{(\text{te})}$. That is, we follow the standard self-supervised procedure for link prediction (as in \Cref{sec:problem-def} and \citet[Appendix B]{cotta2023causal}), and create a self-supervised mask $M \in \{ 0, 1 \}^{N^{\text{(te)}} \times R^{\text{(te)}} \times N^{\text{(te)}}}$ that tunes $\alpha^{(\text{te})}$ to maximize $P(\overline{M} \odot \kgadj^{(\text{te})} \mid M \odot \kgadj^{(\text{te})})$, with $\overline{M} = {\bf 1} - M$.

\section{Experiments} \label{sec:experiments}
In this section, we empirically evaluate our model in predicting missing triplets involving new relation types. \revision{Due to space constraints, we present our main results, and defer readers to \Cref{app:dataset,app:exp-details,app:exp}.}
We set to address the following main questions: 
\begin{enumerate*}[label=\textbf{Q\arabic*},leftmargin=*]
    \item \revision{\textit{Does our model outperform the baselines on a synthetic dataset constructed to contain multiple tasks?}}
    \item \revision{\textit{How does our model compare to the baselines on real-world datasets that likely contain multiple tasks?}}
\end{enumerate*}

\textbf{Baselines.}
We evaluate our model against four baselines: InGram~\citep{lee2023ingram}, ISDEA~\citep{kg-equivariance}, the homogeneous version of NBFNet \citep{zhu2021neural} (NBFNet-homo), and the homogeneous version of ISDEA (ISDEA-homo). 
The homogeneous models are obtained by modifying the corresponding base models to treat all relation types equally. As a result, when predicting a tail node $v$ given a head node $u$ and a relation type $r$, a homogeneous model returns the node $v$ for which the edge $(u, v)$ is most likely to exist, regardless of the relation type. When predicting the relation type $r$ between given nodes $u$ and $v$, a homogeneous model returns a uniform prediction over all possible relation types. This modification allows NBFNet to generalize to new test relation types, a task it cannot perform otherwise. To the best of our knowledge, these models are the only ones applicable to our scenario.

\begin{wraptable}[17]{r}{0.55\textwidth}
    \vspace{-5pt}
    \footnotesize
    \centering
    \caption{Dual-sampling metrics on \textsc{Metafam}. We report mean and std over 3 seeds, with best values in bold, second-best underlined. $\hat{K}$ denotes the maximum number of tasks the architecture can model. 
            \textbf{Our model MTDEA with $\hat{K} = 2$ is comparable to the baselines on Hits@10 and outperforms them on Hits@1.}}\label{tab:metafam-result}
    \resizebox{.9\linewidth}{!}{
    \begin{tabular}{lrr}
        \toprule
        Models & {Hits@1 $\uparrow$} & {Hits@10 $\uparrow$} \\
        \midrule
        NBFNet-homo                    & 0.068 (0.001) & 0.400 (0.001)  \\
        ISDEA-homo                     & 0.000 (0.000) & 0.000 (0.000)  \\
        ISDEA~\citep{kg-equivariance}  & \underline{0.292} (0.029) & 0.609 (0.050) \\
        InGram~\citep{lee2023ingram}   & 0.222 (0.029) & \textbf{0.719} (0.135) \\
        \midrule
        MTDEA ($\hat{K}=2$)            & \textbf{0.344} (0.067) & \underline{0.704} (0.072)  \\
        MTDEA ($\hat{K}=4$)            & 0.172 (0.134) & 0.358 (0.199)  \\
        MTDEA ($\hat{K}=6$)            & 0.169 (0.057) & 0.520 (0.117)  \\
        \bottomrule
    \end{tabular}
    }
\end{wraptable}

\textbf{Dual-sampling metrics.} \label{par:dual-sampling-metrics} 
In line with the dual-sampling loss we proposed in \Cref{eq:task-loss}, we present the \textit{dual-sampling metrics}, which include Hits@$k$, $k \in \{1,10\}$ (Mean Rank (MR), and Mean Reciprocal Rank (MRR) in the Appendix).
For each positive triplet we generate 24 negative samples by corrupting the tail entity and 26 negative samples by corrupting the relation type. These metrics are better suited for measuring the capabilities of the models in our tasks (\Cref{app:metrics}).

\textbf{\textbf{A1}: Synthetic multi-task dataset.}
To address \textbf{Q1},
we construct a synthetic dataset \textsc{MetaFam} that explicitly exhibits conflicting predictive patterns, or multiple tasks, in the \namekgs. In particular, we recreate the conflicting predictive patterns shown in \Cref{fig:multi-task} where we use family relationships in train, and academic relationships in test (see \Cref{app:synthetic-dataset} for details). \Cref{tab:metafam-result} shows the results under the dual-sampling metrics. As we can see from the table, our MTDEA model with two task partitions $\hat{K}=2$ obtains the best performance under 
Hits@1 and achieves comparable performance to the best-performing baseline InGram on Hits@10. In addition, on both Hits@1 and Hits@10, our model surpasses ISDEA, the baseline model that our MTDEA is built upon.
This observation conforms to our expectation as \textsc{MetaFam} was constructed to include exactly two conflicting predictive patterns and therefore a model capable of modeling two distinct tasks is expected to obtain the superior predictions in this dataset. We note that our model falls short of InGram on Hits@10, but this is likely due to the poor performance of ISDEA on this metric. 

\begin{table}
    \centering
    \caption{Dual-sampling metrics on \textsc{WikiTopics-MT1} and \textsc{WikiTopics-MT2}, tested on four topics (\textsc{Health} and \textsc{Taxonomy} for \textsc{WikiTopics-MT1}, \textsc{Location} and \textsc{Science} for \textsc{WikiTopics-MT2}) not seen in training. We report mean and std over 3 seeds, with best values in bold, second-best underlined. $\hat{K}$ denotes the maximum number of tasks the architecture can model. \textbf{Our models consistently outperform the baselines with significantly smaller standard deviations on Hits@1.}} \label{tab:wikitopics-result}
    \vspace{5pt}
    \resizebox{\linewidth}{!}{
    \begin{tabular}{lrrrrrrrrrr}
        \toprule
         & \multicolumn{2}{c}{\textsc{MT1-Health}} & \multicolumn{2}{c}{\textsc{MT1-Taxonomy}} & \multicolumn{2}{c}{\textsc{MT2-Location}} & \multicolumn{2}{c}{\textsc{MT2-Science}} \\
        \cmidrule(lr){2-3} \cmidrule(lr){4-5} \cmidrule(lr){6-7} \cmidrule(lr){8-9} 
        Models & {Hits@1 $\uparrow$} & {Hits@10 $\uparrow$} & {Hits@1 $\uparrow$} & {Hits@10 $\uparrow$} & {Hits@1 $\uparrow$} & {Hits@10 $\uparrow$} & {Hits@1 $\uparrow$} & {Hits@10 $\uparrow$} \\
        \midrule
         NBFNet-homo                   & 0.041 (0.000) & 0.339 (0.003) & 0.034 (0.000) & 0.315 (0.001) & 0.035 (0.002) & 0.292 (0.016) & 0.024 (0.001) & 0.235 (0.008)  \\
        ISDEA-homo                     & 0.000 (0.000) & 0.000 (0.000) & 0.000 (0.000) & 0.000 (0.000) & 0.000 (0.000) & 0.000 (0.000) & 0.000 (0.000) & 0.000 (0.000)  \\
        ISDEA~\citep{kg-equivariance}  & 0.323 (0.140) & 0.481 (0.138) & 0.269 (0.063) & 0.365 (0.082) & 0.393 (0.038) & 0.569 (0.005) & 0.390 (0.034) & 0.617 (0.023)   \\
        InGram~\citep{lee2023ingram}   & 0.122 (0.037) & \textbf{0.869} (0.117) & 0.198 (0.100) & \textbf{0.723} (0.220) & 0.091 (0.043) & \textbf{0.780} (0.049) & 0.063 (0.053) & \textbf{0.691} (0.052)  \\
        \midrule
        MTDEA ($\hat{K}=2$)            & 0.358 (0.191) & 0.513 (0.112) & \underline{0.330} (0.157) & 0.457 (0.121) & \underline{0.417} (0.023) & 0.557 (0.014) & \underline{0.406} (0.008) & \underline{0.595} (0.027)  \\
        MTDEA ($\hat{K}=4$)            & \underline{0.390} (0.127) & 0.496 (0.108) & 0.307 (0.203) & 0.417 (0.175)  & 0.405 (0.049) & 0.547 (0.037) & \textbf{0.409} (0.004) & 0.590 (0.010) \\
        MTDEA ($\hat{K}=6$)           & \textbf{0.457} (0.012) & \underline{0.555} (0.010) & \textbf{0.422} (0.010) & \underline{0.504} (0.025) & \textbf{0.431} (0.004) & \underline{0.558} (0.014) & 0.405 (0.012) & 0.580 (0.024)   \\
        \bottomrule
    \end{tabular}
    }
    \vspace{-12pt}
\end{table}

\textbf{\textbf{A2}: Real-world multi-task datasets.} \label{par:wikitopics-mt-results}
To address \textbf{Q2}, we create two novel multi-task scenarios, named \textsc{WikiTopics-MT1} and \textsc{WikiTopics-MT2}, in the \textsc{WikiTopics} dataset introduced by \citet{kg-equivariance} and obtained from the \textsc{WikiData5M} \cite{wang2021kepler} by grouping the relation types into different topics, such as Art, Education, and Sports.
The \textsc{WikiTopics} dataset was employed by \citet{kg-equivariance} to assess the extrapolation performance of the double-equivariant model, ISDEA, when trained on one topic and tested on another one. For our multi-task datasets \textsc{WikiTopics-MT1} and \textsc{WikiTopics-MT2}, we select for training pairs of topics where, as shown in \citet{kg-equivariance}, ISDEA exhibits the lowest transfer-topic performance (e.g., the \textsc{Art} and \textsc{People}, which we use in \textsc{WikiTopics-MT1}) and test on a third topic (\textsc{Health} or \textsc{Taxonomy}, used in \textsc{WikiTopics-MT1}) on which the ISDEA trained on one training topic (e.g.\ \textsc{Art}) performs good but the ISDEA trained on the other training topic (e.g.\ \textsc{People}) performs poorly. Since the transfer-topic performance of ISDEA indicates the degree of double-exchangeability between graphs in the two topics, which is related to the definition of tasks (\Cref{def:tasks}), this strategy likely produces train and test graphs with multiple tasks.

\Cref{tab:wikitopics-result} presents the performance on \textsc{WikiTopics-MT1} when trained on relations from both \textsc{Art} and \textsc{People} topics and tested on either \textsc{Health} or \textsc{Taxonomy} topic, and on the \textsc{WikiTopics-MT2} when trained on relations from both \textsc{Sport} and \textsc{Health} topics and tested on either \textsc{Location} or \textsc{Science} topic.
Our model MTDEA, and in particular the one with $\hat{K} = 6$, 
\revision{
outperforms ISDEA on both test scenarios across all metrics and surpasses InGram on the Hits@1 metric, while having significantly smaller standard deviations on both Hits@1 and Hits@10 metrics. 
These results suggest that the \textsc{WikiTopics-MT} datasets indeed possess a complicated multi-task structure, and our model, which has the best multi-task modeling capabilities, exhibit more consistent performances.
We note that MTDEA is outperformed by InGram on Hits@10 metric. However, this can be likely attributed to the comparatively low performance of ISDEA, the model on which MTDEA is built, on this metric.}

\vspace{-10pt}
\section{Conclusions}
In this work we studied the problem of extrapolating to new relation types in link prediction tasks in discrete attributed multigraphs. To overcome the challenge faced by existing work when the graphs contain relation types exhibiting contradictory predictive patterns, we proposed a relaxation of the double equivariance models of \citet{kg-equivariance} and demonstrated that this relaxation can be interpreted within a multi-task framework. We designed an architecture capable of modeling this multi-task double equivariance, along with a test-time adaptation procedure to learn task assignments for new relation types. To empirically evaluate our method, we introduced new benchmark datasets featuring multi-task structures and presented novel evaluation metrics to measure its benefits.

\subsubsection*{Acknowledgments}
The authors would like to thank Jianfei Gao and Yangze Zhou for insightful discussions. This work was supported in part by the National Science Foundation (NSF) awards CAREER IIS-1943364, CCF-1918483, and CNS-2212160 and an Amazon Research Award.
Any opinions and findings expressed in this manuscript are those of the authors and do not necessarily reflect the views of the sponsors.


\nocite{cloudbank}

\bibliography{main,more-references}
\bibliographystyle{plainnat}


\newpage
\appendix
\begin{Large}
    \begin{center}
        Supplementary Material for \\
        \textbf{A Multi-Task Perspective for Link Prediction \\
        with New Relation Types and Nodes}
    \end{center}
\end{Large}

\section{Expanded Related Work} \label{app:related-work}
\paragraph{GNNs for \namekg completion.} Due to their recent success in diverse graph-learning tasks, Graph Neural Networks (GNNs) have been widely used to predict missing attributed links between nodes in \namekgs. One of the first adaptation of standard GNNs to multi-relational data was proposed in~\citet{schlichtkrull2018modeling}, while an alternative formulation has been considered in \citet{vashishth2020somposition}. These two models have then inspired several improved versions for both transductive~\citep{galkin2020message,yu2020generalized,zhang2022rethinking} and inductive~\citep{teru2020inductive,zhu2021neural,ali2021improving,zhang2022knowledge} link prediction tasks on \namekgs, even in the context of large graphs~\citep{galkin2022nodepiece}. Recently, their limitations and their relationships have been studied from a theoretical viewpoint, by relating their capabilities in distinguishing different \namekgs to the Weisfeiler-Leman algorithm~\citep{barcelo2022weisfeiler}. These methods, however, cannot work when presented with new relation types in test, which instead represents the main interest of our work. 

\paragraph{Tensor Factorization.} Factorization-based methods~\citep{nickel2011three,bordes2013translating,wang2014knowledge,yang2015embedding,nickel2016holographic,trouillon2016complex,lacroix2018canonical,dettmers2018convolutional,nguyen2018novel,sun2018rotate,Chen2021RelationPA} are classical graph representation learning methods for attributed graphs. Despite their superior empirical performance on transductive tasks, especially when coupled with specific training strategies \citep{ruffinelli2020You,jain2020knowledge}, these models cannot be applied to inductive tasks featuring new nodes in test. To overcome this limitation, \citet{chen2022refactor} propose a new architecture that borrows principles from GNNs and bridges the gap between these two approaches. All these methods, however, are not applicable to the tasks of our interest, where test graphs contain both new nodes and relation types.

\paragraph{Logical reasoning.} Predicting missing attributed links in \namekg can also be performed by learning logical rules that are then used to infer the missing links. \citet{yang2015embedding,yang2017differentiable,sadeghian2019drum,chen2022rlogic} focus on learning Horn clauses from the graph. To understand the expressive power of standard GNNs in learning logical rules, \citet{barcelo2020logical} characterize the fragment of FOC$_2$ formulas, a well-studied
fragment of first order logic, that can be expressed as GNNs. Recently, \citet{qiu2023logical} extended the analysis to heterogeneous graphs.

\paragraph{Zero-shot learning for link prediction in \namekgs.} To predict links involving completely new relation types at test time, zero-shot methods
\revision{typically}
require additional information encoding the semantic of the relation types. \citet{qin2020generative} rely on semantic features obtained from the text descriptions of the relation types. \citet{geng2021ontozsl} enrich the relation features using information from the ontological schema. Finally,
\citet{li2022hierarchical} use the character n-gram information from
the relation name to generate more expressive representations of the relations. As we do not assume access to any extra information apart from the input graphs, not even the relation textual names\footnote{We always consider relation types as numbers, $\cR \coloneqq \{1, \ldots, R\}$, \revision{$R \in \mathbb{N}$}.}, these methods are inapplicable to our scenario.
\revision{
To the best of our knowledge, InGram~\citep{lee2023ingram} and ISDEA~\citep{kg-equivariance} are the only methods capable of generalizing to new test graphs without extra information.
InGram~\citep{lee2023ingram} introduced a method for learning relation embeddings within a relation graph specifically designed to capture the structural affinity between relation types. 
Later, ISDEA~\citep{kg-equivariance} introduced the concept of double-exchangeability, which includes exchangeability between relation types, intuitively their property of being interchangeable with one another, and 
proposed the double equivariant representations capable of generalizing to unseen relations. It also
proved that 
InGram in fact produces positional embeddings that are double equivariant in distribution.
}

\paragraph{Few-shot learning for link prediction in \namekgs.} To the best of our knowledge, most few-shot methods predict novel relation types in test following a meta-learning paradigm~\citep{xiong2018one,chen2019meta,zhang2020few,sun2021one,qian2022few}. 
\revision{
For instance, GMatching~\citep{xiong2018one} proposes to solve the one-shot relation prediction problem by matching the similarity of the new relation type to those seen in training. FSRL~\citep{zhang2020few} extends GMatching to the few-shot setting by using an attention aggregation so that information from all support triplets of the new relation type can be utilized. 
MetaR~\citep{chen2019meta} computes a meta representation for relation types by averaging all node pair-specific relation representations.
CSR~\citep{qian2022few}, on the other hands, matches the test relations to the training ones by comparing the connection subgraphs surrounding the target triplets generated by a hypothesis testing procedure. 
Even though these methods are capable of reasoning over new relation types, and some, such as CSR~\citep{qian2022few}, are also capable of handling new nodes, they all require the presence of a shared observable graph between training and test. In other words, the few-shot triplets of the new relation types and new nodes need be connected to the existing ones already observed in training. 
Hence, they are more constrained than our method since we consider a completely new graph at test time involving no nodes and relations seen during training.
}

\section{Theoretical Analysis} \label{app:proof}

\equivalence*
\begin{proof} We prove each property separately.

    The exchangeability between relation types $\sim_e$ is reflexive. This can be trivially shown by considering the identity node permutation $\text{Id}_N \in \sS_N$ and identity relation type permutation $\text{Id}_{R} \in \sS_R$. Namely, let $\kgadj \in \kgspace$ be a random variable representing an attributed graph sampled from some data distribution. For any relation $r \in \cR$, naturally $\text{Id}_{R} \circ r = r$ and $\text{Id}_{R} \circ \text{Id}_{N} \circ \kgadj = \kgadj$, and consequently $P(\text{Id}_{R} \circ \text{Id}_{N} \circ \kgadj) = P(\kgadj)$. Hence, $r$ is exchangeable with $r$.

    The exchangeability between relation types $\sim_e$ is symmetric. We first note that the node permutation and relation type permutation are commutative \cite{kg-equivariance}. That is, given any attributed graph $\kgadj \in \kgspace$, $\pi \in \sS_N$, and $\sigma \in \sS_R$, we have $\sigma \circ \pi \circ \kgadj = \pi \circ \sigma \circ \kgadj$. In other words, it makes no difference whether we permute the nodes first or we permute the relation types first. Now, for any two relations $r, r' \in \cR$, if $r$ is exchangeable with $r'$, then we know there exists some $\pi \in \sS_N$ and $\sigma \in \sS_R$ such that $\sigma \circ r = r'$ and $P(\kgadj) = P(\sigma \circ \pi \circ \kgadj)$ for any $\kgadj$ sampled from the data distribution. Since $\sS_N$ and $\sS_R$ are groups, $\pi$ and $\sigma$ have unique inverses $\pi' \in \sS_N$ and $\sigma' \in \sS_R$ satisfying $\pi' \circ \pi = \text{Id}_N$ and $\sigma' \circ \sigma = \text{Id}_R$ respectively. Hence,
    \begin{align*}
        &\sigma' \circ r' = \sigma' \circ (\sigma \circ r) = (\sigma' \circ \sigma) \circ r = \text{Id}_N \circ r = r \\
        &\sigma' \circ \pi' \circ (\sigma \circ \pi \circ \kgadj) = \sigma' \circ (\pi' \circ \pi) \circ \sigma \circ \kgadj = \sigma' \circ \sigma \circ \kgadj = \kgadj .
    \end{align*}
    Consequently, we have $P(\sigma \circ \pi \circ \kgadj) = P(\kgadj) = P(\sigma' \circ \pi' \circ (\sigma \circ \pi \circ \kgadj))$. Moreover, since permutations are bijective mappings, we know that for any $\kgadj' \in \kgspace$ there exists some $\kgadj$ such that $\kgadj' = \sigma \circ \pi \circ \kgadj$. Hence, $P(\kgadj') = P(\sigma \circ \pi \circ \kgadj) = P(\sigma' \circ \pi' \circ (\sigma \circ \pi \circ \kgadj)) = P(\sigma' \circ \pi' \circ \kgadj')$ for any $\kgadj'$ sampled from the data distribution. Therefore, $r'$ is also exchangeable with $r$. 

    The exchangeability between relation types $\sim_e$ is transitive. Let $r_1, r_2, r_3 \in \cR$ be three relation types such that $r_1$ is exchangeable with $r_2$, and $r_2$ is exchangeable with $r_3$. Then, there exists some $\pi_1, \pi_2 \in \sS_N$ and $\sigma_1, \sigma_2 \in \sS_R$ such that
    \begin{align*}
        \sigma_1 \circ r_1 = r_2 \qquad &\text{and} \qquad P(\kgadj) = P(\sigma_1 \circ \pi_1 \circ \kgadj) \\
        \sigma_2 \circ r_2 = r_3 \qquad &\text{and} \qquad P(\kgadj') = P(\sigma_2 \circ \pi_2 \circ \kgadj') ,
    \end{align*}
    for any $\kgadj$ and $\kgadj'$ sampled from the data distribution. Hence, take any $\kgadj \in \kgspace$,
    \begin{align*}
        P(\kgadj) 
        &= P(\sigma_1 \circ \pi_1 \circ \kgadj) = P(\sigma_2 \circ \pi_2 \circ (\sigma_1 \circ \pi_1 \circ \kgadj)) \\
        &= P((\sigma_2 \circ \sigma_1) \circ (\pi_2 \circ \pi_1) \circ \kgadj),
    \end{align*}
    where we also have $(\sigma_2 \circ \sigma_1) \circ r_1 = r_3$, showing that $r_1$ is exchangeable with $r_3$. 

    Since the exchangeability between relation types is reflexive, symmetric, and transitive, it is an equivalence relation on $\cR$.
\end{proof}

\section{Datasets Construction} \label{app:dataset}

\subsection{\textsc{WikiTopics-MT}}
The \textsc{WikiTopics-MT} scenarios are derived from the \textsc{WikiTopics} dataset previously introduced by \citet{kg-equivariance}, which comprises 11 \namekgs with relation types in each \namekg corresponding to a specific topic. Our goal is to construct training graphs exhibiting multi-task structures while ensuring the test graph possesses a task also present in training. In \citet{kg-equivariance}, this dataset was leveraged to assess the ISDEA model's zero-shot generalization capabilities on pairs of \namekgs that may have relation types not exchangeable with each other. We observe that the ISDEA model's performance can be viewed from an alternative perspective: poor test performance on a certain topic when trained on a different topic indicates that the topics contain relation types that are likely not exchangeable. Consequently, the corresponding \namekgs likely contains distinct tasks (\Cref{def:tasks}), implying that combining the \namekgs of these two topics would yield an aggregated \namekg containing multiple tasks.

\begin{figure}
    \centering
    \begin{subfigure}[b]{0.48\textwidth}
        \centering
        \includegraphics[width=\textwidth]{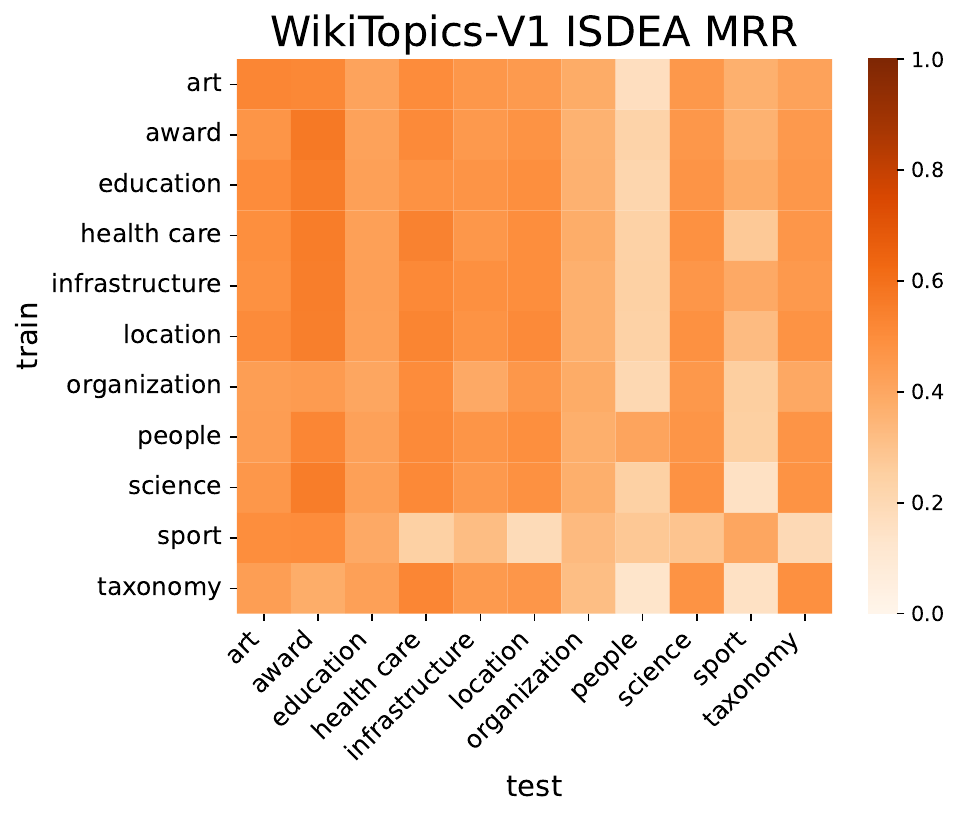}
        \caption{\textsc{WikiTopics-v1}.}
        \label{fig:isdea-transfer-wiki-v1}
    \end{subfigure}
    \hfill
    \begin{subfigure}[b]{0.48\textwidth}
        \centering
        \includegraphics[width=\textwidth]{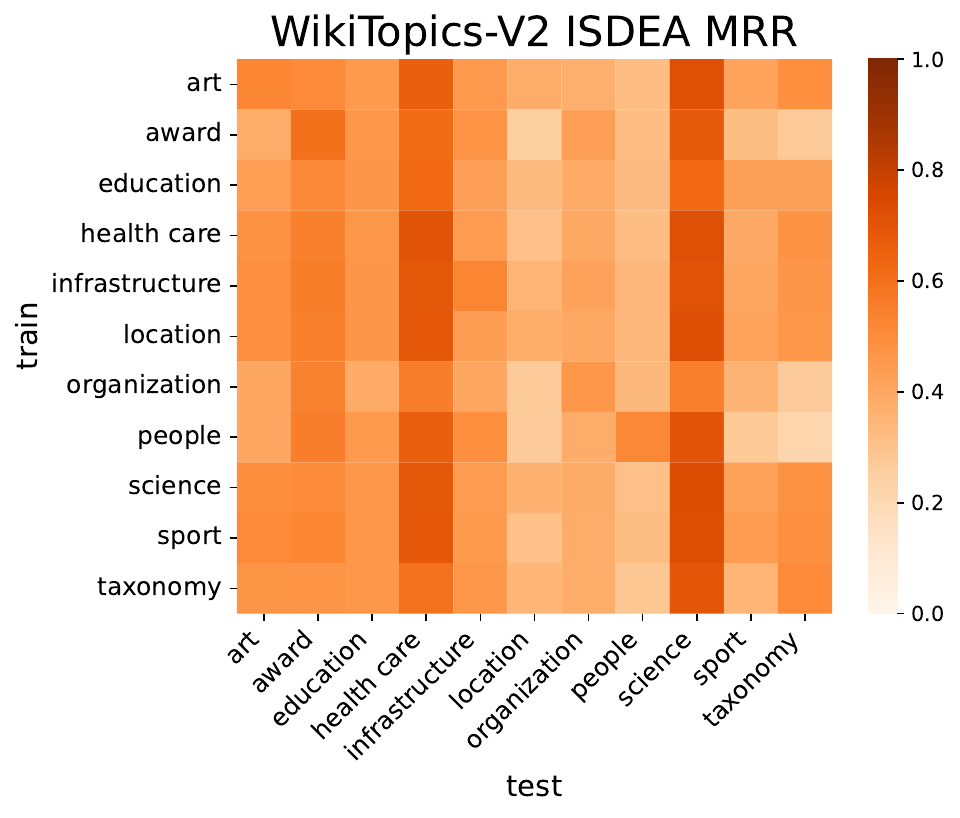}
        \caption{\textsc{WikiTopics-v2}.}
        \label{fig:isdea-transfer-wiki-v2}
    \end{subfigure}
    \caption{ISDEA \cite{kg-equivariance} transfer-topics performance on both versions of \textsc{WikiTopics} without the shortest-distance heuristic embeddings. Color indicates the value of the dual-sampling MRR metric.}
    \label{fig:isdea-transfer-wiki}
\end{figure}

To identify which pairs of topics are likely to contain distinct tasks, we rerun the ISDEA model from \citet{kg-equivariance} on both versions of the \textsc{WikiTopics} dataset. To reduce the memory footprint and the computational time, we run the experiments without the shortest-distance heuristic embeddings that are used to augment the triplet representations in \citet{kg-equivariance} (as explained in \Cref{app:exp-details}). 
\Cref{fig:isdea-transfer-wiki-v1,fig:isdea-transfer-wiki-v2} show the heatmaps representing the transfer-topic performance of the ISDEA model on the 121 pairs of topics for each version of the dataset.  Each row in \Cref{fig:isdea-transfer-wiki-v1,fig:isdea-transfer-wiki-v2} corresponds to one training topic, each column corresponds to a test topic, and the color represents the performance evaluated using the dual-sampling Mean Reciprocal Ranks (MRR). 

Based on the heatmaps (\Cref{fig:isdea-transfer-wiki-v1,fig:isdea-transfer-wiki-v2}), we devise multi-task scenarios using the outlined strategy. For the training graph, we pick two topics such that when trained on one and evaluated on the other the performance is low (e.g. \textsc{Health} and \textsc{Sport} in \Cref{fig:isdea-transfer-wiki-v1}). This indicates that the ISDEA model, trained on one topic (\textsc{Health}), demonstrates relatively poor zero-shot generalization performance on the other topic (\textsc{Sport}), suggesting that the \namekgs associated with these topics likely contain distinct tasks. 
We then combine the \namekgs of the two topics to create an aggregated graph (\textsc{Health} + \textsc{Sport}), which is expected to exhibit multi-task structures. 
Next, we select one test topic for each training topic, and evaluate separately on the two test topics. A test topic (e.g., \textsc{Location}) is determined such that the ISDEA model trained on one of the training topics (in this case, \textsc{Science}) performs well on the selected test topic (\textsc{Location}), but the good performance on this test topic may not necessarily be observed when the ISDEA model is trained on the other training topic (\textsc{Sport}).

Using this data creation strategy, we construct the following four multi-task scenarios, two for each version of \textsc{WikiTopics}:

\begin{enumerate}
    \item \textbf{\textsc{WikiTopics-MT1}}: created from \textsc{WikiTopics-V1}. The training topic is a combination of \textsc{Art} and \textsc{People}, and the 2 test topics are \textsc{Health} and \textsc{Taxonomy}.
    \item \textbf{\textsc{WikiTopics-MT2}}: created from \textsc{WikiTopics-V1}. The training topic is a combination of \textsc{Sport} and \textsc{Health}, and the 2 test topics are \textsc{Location} and \textsc{Science}.
    \item \textbf{\textsc{WikiTopics-MT3}}: created from \textsc{WikiTopics-V2}. The training topic is a combination of \textsc{People} and \textsc{Taxonomy}, and the 2 test topics are \textsc{Art} and \textsc{Infrastructure}.
    \item \textbf{\textsc{WikiTopics-MT4}}: created from \textsc{WikiTopics-V2}. The training topic is a combination of \textsc{Location} and \textsc{Organization}, and the 2 test topics are \textsc{Health} and \textsc{Science}.
\end{enumerate}

\begin{table}
    \centering
    \caption{Dataset statistics of the \textsc{WikiTopics-MT} scenarios.}
    \label{tab:wikitopics-mt-stats}
    \resizebox{\linewidth}{!}{
    \begin{tabular}{llcccc}
        \toprule
        Scenarios & Train/Test Topics & \# Entities & \# Relation Types & \# Observable Triplets & \# Missing Triplets \\
        \midrule
        \multirow{3}{*}{\textsc{WikiTopics-MT1}}  & \textsc{Art} + \textsc{People} (Train) & 10000 & 35 & 32145 & 3571 \\
                                                 & \textsc{Health} (Test)                 & 10000 & 8  & 14110 & 1566 \\
                                                 & \textsc{Taxonomy} (Test)               & 10000 & 10 & 16526 & 1834 \\
        \midrule
        \multirow{3}{*}{\textsc{WikiTopics-MT2}} & \textsc{Sport} + \textsc{Health} (Train) & 10000 & 19 & 43528 & 4836 \\
                                                 & \textsc{Location} (Test)                 & 10000 & 11 & 22971 & 2552 \\
                                                 & \textsc{Science} (Test)                  & 10000 & 17 & 14852 & 1650 \\              
        \midrule
        \multirow{3}{*}{\textsc{WikiTopics-MT3}} & \textsc{People} + \textsc{Taxonomy} (Train) & 10000 & 99 & 57815 & 6243 \\
                                                 & \textsc{Art} (Test)                         & 10000 & 65 & 28023 & 3113 \\
                                                 & \textsc{Infrastructure} (Test)              & 10000 & 37 & 21646 & 2405 \\  
        \midrule
        \multirow{3}{*}{\textsc{WikiTopics-MT4}} & \textsc{People} + \textsc{Taxonomy} (Train) & 10000 & 94 & 54140 & 6015 \\
                                                 & \textsc{Location} (Test)                    & 10000 & 62 & 80269 & 8918 \\
                                                 & \textsc{Organization} (Test)                & 10000 & 34 & 30214 & 3357 \\
        \bottomrule
    \end{tabular}
    }
\end{table}

\Cref{tab:wikitopics-mt-stats} shows the datasets statistics of the 4 multi-task scenarios. The experiment results of the \textsc{WikiTopics-MT1} scenario are shown in \Cref{par:wikitopics-mt-results} in the main paper, and the additional experiment results of the \textsc{WikiTopics-MT2}, \textsc{WikiTopics-MT3}, and \textsc{WikiTopics-MT4} scenarios are shown in \Cref{app:wikitopics-mt}.

\subsection{\textsc{FBNELL}}
We create the \textsc{FBNELL} dataset by combining the \textsc{FB15K-237}~\cite{schlichtkrull2018modeling} and the \textsc{NELL-995}~\cite{Nell995} datasets. 
The training graph is obtained by first choosing the top 50 most frequent relation types in each \textsc{FB15K-237} and \textsc{NELL-995}, yielding a total of 100 relation types, and then extracting the triplets corresponding to these 100 relation types. 
For the test graph, we pick the top 100 most frequent relation types from each dataset, extract the triplets corresponding to the resulting 200 relation types, and predict only those triplets that involve new relation types, while using the remaining triplets as the observable (test) graph. Consequently, the test graph's set of relation types forms a strict superset of those in the training graph, but the evaluation is performed only on the unseen ones. \Cref{tab:fbnell-stats} shows the statistics of the dataset.

We emphasize that, unlike in the data construction strategy employed for the \textsc{WikiTopics-MT} scenarios, we do not actively identify the presence of multiple tasks within either \textsc{FB15K-237} or \textsc{NELL-995}, nor verify whether their combination exhibit a multi-task structure.
Therefore, even though the two datasets come from distinct domains, they might still share the same relational task (single-task). 

\begin{table}
    \centering
    \caption{Dataset statistics of the \textsc{FBNELL} dataset.}
    \label{tab:fbnell-stats}
    \resizebox{\linewidth}{!}{
    \begin{tabular}{lcccc}
        \toprule
        Train/Test Splits & \# Entities & \# Relation Types & \# Observable Triplets & \# Missing Triplets \\
        \midrule
        Train & 4797 & 100 & 10275 & 1224 \\
        Test  & 4725 & 200 & 10685 & 597 \\
        \bottomrule
    \end{tabular}
    }
\end{table}

\subsection{\textsc{MetaFam}}  \label{app:synthetic-dataset}
We construct a synthetic dataset, dubbed \textsc{MetaFam}, that explicitly exhibits conflicting predictive patterns, or equivalently, a multi-task structure. In particular, we recreate the conflicting predictive patterns shown in \Cref{fig:multi-task}.
We generate the dataset by first creating the family trees using the ontology and the code provided in \citet{hohenecker2020ontology}. Each family tree is generated by starting from a single person and incrementally adding a new child to an existing node until the tree reaches the maximum size of 26 nodes or a maximum depth of 5. The parent of the node to be added is chosen uniformly at random, with the only constraint that the maximum branching factor of each tree is 5. Each family tree contains triplets involving 29 different relation types representing different kinds of relationships, such as \texttt{mother\_of}, \texttt{daughter\_of}, \texttt{uncle\_of}.

We generate the training split by randomly selecting 50 non-isomorphic family trees. In each training family tree we mask out either some of the triplets with relation types \texttt{mother\_of} and \texttt{father\_of}, or some of the triplets involving relation types \texttt{son\_of} and \texttt{daughter\_of}, and we use those as the triplets we aim to predict during training time. Doing so ensures that the model is challenged with the two conflicting patterns illustrated in \Cref{fig:multi-task} when learning to predict these triplets. Specifically, the former two relation types, \texttt{mother\_of} and \texttt{father\_of}, obey to the first kind of predictive pattern (e.g. subgraph $\kgadj'_1$ as illustrated in \Cref{fig:multi-task}), and the latter two relation types, \texttt{son\_of} and \texttt{daughter\_of}, follow the second kind of predictive pattern (the rest of the training graph $\kgadj_1$ as illustrated in \Cref{fig:multi-task}). As the test split, we create 25 additional non-isomorphic family trees having the same relation types of the training \namekg but permuted. In test, we only mask out triplets corresponding to the (permuted) relation types \texttt{mother\_of} and \texttt{father\_of}, so that only one predictive pattern is required to accurately predict these missing triplets at test time.

\begin{table}
    \centering
    \caption{Dataset statistics of the \textsc{MetaFam} dataset.}
    \label{tab:metafam-stats}
    \resizebox{\linewidth}{!}{
    \begin{tabular}{lcccc}
        \toprule
        Train/Test Splits & \# Entities & \# Relation Types & \# Observable Triplets & \# Missing Triplets \\
        \midrule
        Train & 1316 & 29 & 13630 & 781 \\
        Test  & 656  & 29 & 7257  & 184 \\
        \bottomrule
    \end{tabular}
    }
\end{table}

\Cref{tab:metafam-stats} shows the statistics of the \textsc{MetaFam} dataset. The experiment results are described in the main text and in \Cref{app:synthetic}. 

\section{Implementation and Experiment Details} \label{app:exp-details}

\subsection{Licenses, Computational Resources and Experimental Setup}

We implemented our MTDEA model using PyTorch~\citep{PyTorch} and PyTorch Geometric~\citep{PyG}, which are available under the BSD and MIT license respectively. The Wikidata knowledge base~\citep{Wikidata}, which the \textsc{WikiTopics} dataset is based on, is available under the CC0 1.0 license. We ran our experiments on NVIDIA V100, A100, GeForce RTX 2080Ti, GeForce RTX 4090, and Titan V GPUs. We use Weights \& Biases~\citep{wandb} to perform hyperparameter tuning. We train all models (baselines and MTDEA) in all experiments for a maximum of 10 epochs, with an early stop patience of 5 epochs based on the dual-sampling MRR value on the validation set. At test time we adapt our MTDEA models to learn the task assignments for the test relation types (as described in \Cref{sec:transductive}) for a maximum of 10 epochs. For our MTDEA models, we train with a number of maximum task partitions $\hat{K}=2, 4, 6$ in all experiments. The time spent on each experiment depends mainly on the size of the dataset and on $\hat{K}$. For example, training MTDEA with $\hat{K}=4$ on \textsc{WikiTopics-MT3} takes around 10 hours to complete, while training MTDEA with $\hat{K}=2$ on \textsc{WikiTopics-MT1} takes around 2 hours and 30 minutes.
Our code and datasets are available. \footnote{\url{https://anonymous.4open.science/r/MTDEA}.}

\subsection{Details of the Neural Architecture}

\paragraph{Attention matrix.}
In our implementation, the attention matrix $\alpha$ (\Cref{sec:attn-weights}) is obtained from a real-valued learnable weight matrix $w \in \sR^{R \times \hat{K}}$, where $R$ is the number of relations and $\hat{K}$ the number of maximum partitions we allow. 
We apply a Softmax activation over the task partition dimension for every relation type $r \in \cR$, i.e., $\alpha_{r, k} = \frac{\exp(w_{r, k})}{\sum_{k'=1}^{\hat{K}} \exp(w_{r, k'})}$, for $k \in \{1, \ldots \hat{K}\}$. 
At training time we refer to $w$ as $w^{(\text{tr})}$, since $R$ is $R^{(\text{tr})}$, the number of training relation types.
During the test-time adaptation, we freeze all parameters of the model, we discard $w^{(\text{tr})}$ and initialize a new matrix $w^{(\text{adapt})} \in \sR^{R^{(\text{te})} \times \hat{K}}$, where $R^{(\text{te})}$ is the number of test relation types. Then, $w^{(\text{adapt})}$ is optimized via gradient descent with the same training loss used in training (\Cref{eq:full-loss}), with the only difference that the positive and negative triplets are now sampled from the observable test graph $\kgadj^{\text{(te)}}$.

\paragraph{Structural node representation for link prediction tasks.} Structural node representations, which are obtained from GNNs, are known to have limited capabilities for link prediction tasks in homogeneous graphs~\citep{srinivasan2020on,you2019position}, an issue that also arises in \namekgs. \revision{Theoretically, \Cref{eq:asymcond-eqvlin} overcomes this limitation by employing GNNs that output pairwise-representations as $L_1^{(t)}, L_2^{(t)}, L_3^{(t)}$}. However, most-expressive pairwise representations~\citep{zhang2018link,zhu2021neural,zhang2021labeling} are computationally expensive. In \citet{kg-equivariance}, the authors sought a middle ground for their ISDEA model \revision{by employing structural node representations enhanced with heuristic embeddings}, such as the shortest distances between the two nodes in the pair to be predicted. Specifically, the representation for a triplet $(u, r, v)$, with $u,v \in \cV, r \in \cR$, before the final MLP layers is obtained as $h_{u, r}^{(T)} \| h_{v, r}^{(T)} \| d(u, v) \| d(v, u)$, where $h_{u, r}^{(T)}$ and $h_{u, r}^{(T)}$ are the structural node representations for, respectively, nodes $u$ and $v$ specific to the relation type $r$ obtained after $T$ ISDEA layers; $d(u, v)$ and $d(v, u)$ are the shortest distances from node $u$ to $v$ and from node $v$ to $u$ in the directed \namekg, and $\|$ denotes the vector concatenation operation. 

Nevertheless, computing the shortest distance, $d(u, v)$, for all pairs $(u,v)$, $u,v \in \cV$ in the \namekg is time- and space-demanding. Due to this limitation, we opt \textit{not} to compute the shortest distances, and instead use only the structural node representations as the representation of a triplet $(u, r, v)$, that is $h_{u, r}^{(T)} \| h_{v, r}^{(T)}$, in all our experiments, except for the synthetic \textsc{MetaFam} dataset. \revision{In practice, this means that we implement $L_1^{(t)}, L_2^{(t)}, L_3^{(t)}$ in \Cref{eq:asymcond-eqvlin} as GNNs outputting node representations}. We empirically observe no performance degradation when removing the shortest distance heuristics on real-world \namekgs. 

\paragraph{Layers.}
In all our experiments, excluding those on the synthetic \textsc{MetaFam} dataset, our MTDEA model employs two GNN-based soft \namelayersabbr (\Cref{eq:asymcond-eqvlin}). Conversely, for the synthetic \textsc{MetaFam} dataset, our model consists of only one GNN-based \namelayerabbr, in order to learn exactly the conflicting predictive patterns depicted in \Cref{fig:multi-task}. We select GIN~\citep{xu2018powerful} with $\epsilon = 0$ as our GNN layer, which implements the $L_1$, $L_2$, and $L_3$ components of a \namelayerabbr.

After these GNN-based \namelayersabbr, we employ two soft \namelayersabbr with MLPs $L_1$, $L_2$, and $L_3$ components. The representation of each triplet $(u,r,v)$ with $u,v \in \cV$, $r \in \cR$ is then obtained using the node representations after these layers as $h_{u, r}^{(T)} \| h_{v, r}^{(T)}$, and it is then passed to
a two-layers MLP to obtain the final prediction. 

\subsection{Hyper-parameters}

In all experiments, involving either our MTDEA models or the baseline ISDEA-homo, we train with mini-batches comprising 256 positive triplets. We use a training negative sample rate of 2 for both the tail-based negative samples and the relation-based negative samples. Hence, for each positive triplet in a minibatch, we construct four negative samples, thus resulting in mini-batches containing 1280 triplets in total. Additional hyper-parameters values for MTDEA and ISDEA-homo include hidden layer dimension of 32, ReLU activations, mean set aggregation (i.e. the set aggregation operation within the parentheses of $L_2$ and $L_3$ in \Cref{eq:cond-eqvlin}), a gradient clipping norm of 1.0, a learning rate of 0.001, and a weight decay rate of $5 \times 10^{-4}$ in our Adam optimizer.

Apart from the aforementioned hyper-parameters, our MTDEA features additional hyper-parameters associated with its regularization losses. In all our experiments, we set the initial regularization coefficient values to $\lambda_1 = 0.1$ and $\lambda_2 = 0.1$ (\Cref{eq:full-loss}), with a per-epoch multiplicative annealing factor of $1.1$. This approach ensures that the regularization gains more significance in later epochs (closer to the end). That is, after each epoch, the values of $\lambda_1$ and $\lambda_2$ are updated to $1.1 \times \lambda_1$ and $1.1 \times \lambda_2$ respectively. 

For the NBFNet-homo baseline, we use the default hyper-parameters provided by \citet{zhu2021neural}. Specifically, we choose a hidden dimension of 32 for all layers, the distance multiplier message function, the pna aggregate function, and we employ layer norm. We further use Adam optimizer with a learning rate of $0.005$ and a batch size of 64. 

\revision{
For the InGram~\citep{lee2023ingram} baseline, we tune its hyperparameters on all datasets by performing a grid search over the configuration of ranking loss margin $\gamma \in \{ 1.0, 2.0 \}$, learning rate $\alpha \in \{ 0.0005, 0.001 \}$, number of relation layers $L \in \{ 1, 2, 3 \}$, and number of entity layers $\hat{L} \in \{ 2, 3, 4 \}$. We use the suggested best values from~\citet{lee2023ingram} and default values present in their codebase for all other hyperparameters, such as the number of bins $B = 10$ and the number of attention heads $K=8$. We fix the entity embedding dimension and relation embedding dimension size to 32 for a fair comparison with other models. 
}

\section{Additional Experiments} \label{app:exp}

\subsection{Dual-sampling versus Entity-centric Metrics} \label{app:metrics}

In both our training loss (\Cref{eq:task-loss}) and evaluation metrics (\Cref{par:dual-sampling-metrics}), we adopt a dual-sampling scheme wherein for each positive triplet we draw two different types of negative samples: those with the tail node randomly corrupted (entity-centric) and those with the relation type randomly corrupted (relation-type centric). In contrast, existing literature focuses solely on entity-centric negative samples~\cite{yang2015embedding,schlichtkrull2018modeling,zhu2021neural}. We argue that correctly predicting the tail node, given a head node and a relation type, is as important as determining the type of relation connecting a head and a tail node of interest, and therefore the two negative sampling schemes should be used in conjunction. This combination results in the dual-sampling metrics we propose.

\Cref{tab:wikitopics-mt1-results,tab:wikitopics-health-tax-ent,tab:fbnell-result-both-metrics} compare the performances of various models under the dual-sampling metrics with their performances under the traditional entity-centric metrics, commonly used in the literature. As can be seen from the tables, the baseline NBFNet-homo demonstrates strong performances under the entity-centric metrics, but not under the dual-sampling metrics. In particular, it achieves 95\% entity-centric Hits@10 on the \textsc{FBNELL} dataset. Therefore \textit{a model that disregards the information contained in the relation types can achieve near-perfect accuracy under the entity-centric metrics} (simply by predicting whether $u$ is connected to $v$, irrespective of the relation type).
These results suggest that the entity-centric metrics are insufficient for assessing model performance, as homogeneous link prediction methods can easily solve a task when evaluated based on these metrics.
\revision{
Such an observation is also echoed by~\citet{jambor2021exploring} on few-shot link prediction tasks, in which the authors commented that they found ``a simple zero-shot baseline - which ignores any relation-specific information - achieves surprisingly strong performance," further giving grounds to the importance of evaluating a model's performance not only on predicting entities but also on predicting relation types between nodes.
}
In contrast, the homogeneous models achieve at most 38\% Hits@10 accuracy under the dual-sampling metrics, illustrating that the dual-sampling scheme is a more suitable, comprehensive, and challenging evaluation scheme, where homogeneous link prediction models cannot unreasonably obtain near-perfect performances. 

\begin{table}
    \centering
    \caption{Model performances on \textsc{WikiTopics-MT1} under the \textbf{dual-sampling metrics}, tested on two topics (\textsc{Health} and \textsc{Taxonomy}) not seen in training (\textsc{Art} + \textsc{People}). We report mean and std across 3 random seeds. For our MTDEA, $\hat{K}$ denote the maximum number of tasks the architecture can model (\Cref{sec:attn-weights}).} \label{tab:wikitopics-mt1-results} 
    \resizebox{\linewidth}{!}{
    \begin{tabular}{lrrrrrrrrrr}
        \toprule
         & \multicolumn{4}{c}{Test on \textsc{Health} topic} & \multicolumn{4}{c}{Test on \textsc{Taxonomy} topic} \\
        \cmidrule(r){2-5} \cmidrule(l){6-9} 
        Models & {MR $\downarrow$} & {MRR $\uparrow$} & {Hits@1 $\uparrow$} & {Hits@10 $\uparrow$} & {MR $\downarrow$} & {MRR $\uparrow$} & {Hits@1 $\uparrow$} & {Hits@10 $\uparrow$} \\
        \midrule
         NBFNet-homo                    & 15.843 (0.119)   & 0.122 (0.001) & 0.041 (0.000) & 0.339 (0.003) & \underline{16.260} (0.068) & 0.113 (0.000) & 0.034 (0.000) & 0.315 (0.001)  \\
        IS-DEA-homo                     & 42.276 (0.768) & 0.025 (0.001) & 0.000 (0.000) & 0.000 (0.000)   & 40.963 (1.276) & 0.026 (0.001) & 0.000 (0.000) & 0.000 (0.000)  \\
        IS-DEA~\citep{kg-equivariance}  & 15.405 (6.030)  & 0.384 (0.133) & 0.323 (0.140) & 0.481 (0.138)  & 19.143 (3.895)  & 0.323 (0.063) & 0.269 (0.063) & 0.365 (0.082) \\
        InGram~\citep{lee2023ingram} & \textbf{5.489} (1.701) & 0.342 (0.074) & 0.122 (0.037) & \textbf{0.869} (0.117) & \textbf{7.045} (3.136) & 0.368 (0.113) & 0.198 (0.100) & \textbf{0.723} (0.220) \\
        \midrule
        MTDEA ($\hat{K}=2$)             & 15.015 (4.947)  & \underline{0.422} (0.170) & 0.358 (0.191) & 0.513 (0.112) & 16.840 (4.222)  & \underline{0.393} (0.141) & \underline{0.330} (0.157) & 0.457 (0.121) \\
        MTDEA ($\hat{K}=4$)             & 16.576 (2.293)  & 0.441 (0.115) & \underline{0.390} (0.127) & 0.496 (0.108) & 18.171 (4.242)  & 0.365 (0.191) & 0.307 (0.203) & 0.417 (0.175) \\
        MTDEA ($\hat{K}=6$)             & \underline{13.774} (3.224)  & \textbf{0.504} (0.013) & \textbf{0.457} (0.012) & \underline{0.555} (0.010) & 16.902 (3.837)  & \textbf{0.470} (0.018) & \textbf{0.422} (0.010) & \underline{0.504} (0.025)  \\
        \bottomrule
    \end{tabular}
    }

    \bigskip

    \centering
    \caption{Model performances on \textsc{WikiTopics-MT1} under the \textbf{entity-centric metrics}, tested on two topics (\textsc{Health} and \textsc{Taxonomy}) not seen in training (\textsc{Art} + \textsc{People}). We report mean and std across 3 random seeds. For our MTDEA, $\hat{K}$ denote the maximum number of tasks the architecture can model (\Cref{sec:attn-weights}).} \label{tab:wikitopics-health-tax-ent} 
    \resizebox{\linewidth}{!}{
    \begin{tabular}{lrrrrrrrr}
        \toprule
         & \multicolumn{4}{c}{Test on \textsc{Health} topic} & \multicolumn{4}{c}{Test on \textsc{Taxonomy} topic} \\
        \cmidrule(r){2-5} \cmidrule(l){6-9}
        Models & {MR $\downarrow$} & {MRR $\uparrow$} & {Hits@1 $\uparrow$} & {Hits@10 $\uparrow$} & {MR $\downarrow$} & {MRR $\uparrow$} & {Hits@1 $\uparrow$} & {Hits@10 $\uparrow$} \\
        \midrule
        NBFNet-homo                   & \textbf{5.767} (0.238) & \textbf{0.628} (0.004) & \textbf{0.568} (0.004) & \textbf{0.858} (0.017)  & \textbf{5.980} (0.149) & \textbf{0.572} (0.001) & \textbf{0.498} (0.002) & \textbf{0.855} (0.003) \\
        ISDEA-homo                  & 32.793 (1.683) & 0.112 (0.035) & 0.054 (0.025) & 0.172 (0.065) & 30.065 (2.717) & 0.108 (0.051) & 0.039 (0.028) & 0.203 (0.124) \\
        IS-DEA~\citep{kg-equivariance}   & 19.538 (8.350) & 0.364 (0.136) & 0.314 (0.142) & 0.444 (0.140) & 24.138 (4.299) & 0.301 (0.059) & 0.252 (0.063) & 0.336 (0.071) \\
        InGram~\citep{lee2023ingram} & \underline{9.169} (4.148) & 0.504 (0.085) & 0.397 (0.074) & \underline{0.697} (0.013) & \underline{10.962} (4.586) & \underline{0.457} (0.073) & 0.343 (0.070) & \underline{0.682} (0.107)  \\
        \midrule
        MTDEA ($\hat{K}=2$)              & 19.707 (6.651) & 0.422 (0.166) & 0.360 (0.185) & 0.486 (0.138) & 23.146 (7.251) & 0.368 (0.149) & 0.300 (0.158) & 0.457 (0.119) \\
        MTDEA ($\hat{K}=4$)             & 20.316 (5.700) & 0.442 (0.122) & 0.386 (0.131) & 0.494 (0.115) & 22.408 (6.274) & 0.375 (0.132) & 0.310 (0.141) & 0.449 (0.124) \\
        MTDEA ($\hat{K}=6$)            & 14.920 (1.523) & \underline{0.513} (0.013) & \underline{0.455} (0.023) & 0.570 (0.007) & 19.405 (0.518) & 0.453 (0.002) & \underline{0.392} (0.002) & 0.524 (0.002) \\
        \bottomrule
    \end{tabular}
    }

    \bigskip

    \centering
    \caption{Model performances on \textsc{FBNELL} under the \textbf{dual-sampling metrics} (left) and the \textbf{entity-centric metrics} (right). We report mean and std across 3 random seeds. For our MTDEA, $\hat{K}$ denote the maximum number of tasks the architecture can model (\Cref{sec:attn-weights}).} \label{tab:fbnell-result-both-metrics}
    \resizebox{\linewidth}{!}{
    \begin{tabular}{lrrrr||rrrr}
        \toprule
         & \multicolumn{4}{c}{\textbf{Dual-Sampling Metrics}} & \multicolumn{4}{c}{\textbf{Entity-Centric Metrics}} \\
        \cmidrule(r){2-5} \cmidrule(l){6-9} 
        Models & {MR $\downarrow$} & {MRR $\uparrow$} & {Hits@1 $\uparrow$} & {Hits@10 $\uparrow$} & {MR $\downarrow$} & {MRR $\uparrow$} & {Hits@1 $\uparrow$} & {Hits@10 $\uparrow$} \\
        \midrule
        NBFNet-homo                     & 9.410 (0.029)   & 0.129 (0.001) & 0.042 (0.001) & 0.379 (0.002) & \textbf{2.501} (0.041) & \textbf{0.818} (0.008) & \textbf{0.782} (0.010) & \textbf{0.950} (0.001) \\
        ISDEA-homo                  & 31.625 (0.940)  & 0.033 (0.001) & 0.000 (0.000) & 0.000 (0.000) & 9.424 (1.915) & 0.373 (0.028) & 0.235 (0.025) & 0.700 (0.067) \\
        IS-DEA~\citep{kg-equivariance}  & 10.925 (0.383)  & \textbf{0.624} (0.010) & \textbf{0.562} (0.014) & 0.697 (0.010) & 10.924 (0.893) & 0.611 (0.013) & 0.540 (0.009) & 0.712 (0.030) \\
        InGram~\citep{lee2023ingram} & \textbf{4.258} (0.612) & {0.456} (0.051) & 0.251 (0.052) & \textbf{0.932} (0.037) & \underline{4.258} (0.612) & 0.456 (0.051) & 0.251 (0.052) & \underline{0.932} (0.037) \\
        \midrule
        MTDEA ($\hat{K}=2$)             & \underline{9.106} (0.162)   & \underline{0.622} (0.012) & \underline{0.553} (0.010) & \underline{0.704} (0.024) & 10.797 (0.920) & 0.602 (0.011) & 0.524 (0.010) & 0.707 (0.028) \\
        MTDEA ($\hat{K}=4$)             & 10.730 (0.666)  & 0.606 (0.006) & 0.543 (0.007) & 0.680 (0.023) & 10.464 (0.092) & \underline{0.613} (0.011) & \underline{0.540} (0.015) & 0.720 (0.007) \\
        MTDEA ($\hat{K}=6$)             & 10.386 (0.683)  & 0.609 (0.015) & 0.547 (0.017) & 0.678 (0.008) & 10.753 (0.828) & 0.612 (0.005)  & 0.538 (0.008) & 0.719 (0.021) \\
        \bottomrule
    \end{tabular}
    }
\end{table}

\subsection{Synthetic Experiments} \label{app:synthetic}

We conduct additional experiments using the dataset \textsc{MetaFam}, which was explicitly constructed to exhibit conflicting predictive patterns, or multiple tasks, in the \namekgs (\Cref{app:synthetic-dataset}). \Cref{tab:synth-dual-results} shows the results under the dual-sampling metrics. As we can see from the table, our MTDEA model with two task partitions $\hat{K}=2$ consistently obtains the best performance under 
\revision{all metrics except for MR and Hits@10.}
This observation conforms to our expectation, since the \textsc{MetaFam} was constructed to include exactly two conflicting predictive patterns and therefore a model capable of modeling two distinct tasks is expected to obtain the best predictions in this dataset. 
\revision{
We note that, again as mentioned in~\Cref{par:wikitopics-mt-results}, our model falls short of InGram~\citep{lee2023ingram} on MR and Hits@10, but this is likely due to the poor performance of ISDEA~\citep{kg-equivariance} on these metrics (since our model MTDEA is built on ISDEA). Still, MTDEA outperforms ISDEA on MR and Hits@10 with a significant margin. 
}

We further investigate the performances of the models under different metrics. We consider the entity-centric metrics, where we generate 50 negative samples for each positive triplet by corrupting its tail node, and we additionally compare to what we refer as the relation-type centric metrics, obtained by constructing 50 negative samples for each positive triplet by corrupting its relation type. These results are summarized in \Cref{tab:synth-entity-results,tab:synth-rel-results}. We observe that, although our best-performing model (MTDEA with $\hat{K}=2$) is slightly worse than the baseline ISDEA under the entity-centric metrics, it outperforms 
\revision{all the baseline models under most of the relation-based metrics.}

\begin{table}
    \centering
    \caption{Model performances on \textsc{MetaFam} under the \textbf{dual-sampling metrics}. We report mean and std across 3 random seeds. For our MTDEA, $\hat{K}$ denote the maximum number of tasks the architecture can model (\Cref{sec:attn-weights}). \textit{Our MTDEA is the only model that can correctly represent the existing conflicting patterns.}} \label{tab:synth-dual-results}
    \resizebox{\linewidth}{!}{
    \begin{tabular}{lrrrrrr}
        \toprule
        Models & {MR $\downarrow$} & {MRR $\uparrow$} & {Hits@1 $\uparrow$} & {Hits@3 $\uparrow$} & {Hits@5 $\uparrow$} & {Hits@10 $\uparrow$} \\
        \midrule
        NBFNet-homo                     & 13.785 (0.045) & 0.153 (0.001) & 0.068 (0.001) & 0.145 (0.002) & 0.216 (0.001) & 0.400 (0.001) \\
        ISDEA-homo                      & 27.369 (0.027) & 0.037 (0.000) & 0.000 (0.000) & 0.000 (0.000) & 0.000 (0.000) & 0.000 (0.000)  \\
        IS-DEA~\citep{kg-equivariance}  & 10.112 (1.120)  & \underline{0.292} (0.029) & \underline{0.145} (0.025) & \underline{0.323} (0.038) & \underline{0.440} (0.044) & 0.609 (0.050)  \\
        InGram~\citep{lee2023ingram}    & \textbf{7.805} (1.181) & 0.222 (0.029) & 0.050 (0.033) & 0.222 (0.060) & 0.380 (0.062) & \textbf{0.719} (0.135) \\
        \midrule
        MTDEA ($\hat{K}=2$)             & \underline{9.763} (4.246)  & \textbf{0.344} (0.067) & \textbf{0.178} (0.063) & \textbf{0.407} (0.068) & \textbf{0.518} (0.083) & \underline{0.704} (0.072)  \\
        MTDEA ($\hat{K}=4$)             & 13.795 (3.966)  & 0.172 (0.134) & 0.070 (0.114) & 0.150 (0.173) & 0.213 (0.174) & 0.358 (0.199)  \\
        MTDEA ($\hat{K}=6$)             & 11.332 (0.894)  & 0.169 (0.057) & 0.031 (0.029) & 0.148 (0.110) & 0.272 (0.161) & 0.520 (0.117) \\
        \bottomrule
    \end{tabular}
    }

    \bigskip

    \centering
    \caption{Model performances on \textsc{MetaFam} under the \textbf{entity-centric metrics}. We report mean and std across 3 random seeds. For our MTDEA, $\hat{K}$ denote the maximum number of tasks the architecture can model (\Cref{sec:attn-weights}).} \label{tab:synth-entity-results}
    \resizebox{\linewidth}{!}{
    \begin{tabular}{lrrrrrr}
        \toprule
        Models & {MR $\downarrow$} & {MRR $\uparrow$} & {Hits@1 $\uparrow$} & {Hits@3 $\uparrow$} & {Hits@5 $\uparrow$} & {Hits@10 $\uparrow$} \\
        \midrule
        NBFNet-homo                     & \textbf{1.114} (0.071)  & \textbf{0.952} (0.026) & \textbf{0.939} (0.032) & \textbf{0.987} (0.010) & \textbf{0.998} (0.002) & \textbf{1.000} (0.000) \\
        ISDEA-homo                      & 1.742 (0.056)  & 0.741 (0.031) & 0.558 (0.056) & 0.936 (0.004) & 0.991 (0.002) & 0.998 (0.003) \\
        IS-DEA~\citep{kg-equivariance}  & 1.663 (0.069)  & 0.757 (0.015) & 0.580 (0.021) & 0.941 (0.020) & 0.996 (0.006) &  \textbf{1.000} (0.000) \\
        InGram~\citep{lee2023ingram} & \underline{1.552} (0.085) & \underline{0.788} (0.034) & \underline{0.626} (0.060) & \underline{0.960} (0.010) & \underline{0.997} (0.002) & \textbf{1.000} (0.000) \\
        \midrule
        MTDEA ($\hat{K}=2$)             & 4.226 (4.363)   & 0.697 (0.082) & 0.527 (0.052) & 0.870 (0.118) & 0.913 (0.139) & 0.919 (0.140)  \\
        MTDEA ($\hat{K}=4$)             & 1.812 (0.032)  & 0.710 (0.021) & 0.505 (0.046) & 0.935 (0.014) & 0.989 (0.005) & \textbf{1.000} (0.000) \\
        MTDEA ($\hat{K}=6$)             & 1.685 (0.123)  & 0.748 (0.039) & 0.564 (0.063) & 0.950 (0.010) & 0.995 (0.004) & \textbf{1.000} (0.000) \\
        \bottomrule
    \end{tabular}
    }

    \bigskip

    \centering
    \caption{Model performances on \textsc{MetaFam} under the \textbf{relation-type centric metrics}. We report mean and std across 3 random seeds. For our MTDEA, $\hat{K}$ denote the maximum number of tasks the architecture can model (\Cref{sec:attn-weights}).} \label{tab:synth-rel-results}
    \resizebox{\linewidth}{!}{
    \begin{tabular}{lrrrrrr}
        \toprule
        Models & {MR $\downarrow$} & {MRR $\uparrow$} & {Hits@1 $\uparrow$} & {Hits@3 $\uparrow$} & {Hits@5 $\uparrow$} & {Hits@10 $\uparrow$}  \\
        \midrule
        NBFNet-homo                     & 51.000 (0.000) & 0.020 (0.000) & 0.000 (0.000) & 0.000 (0.000) & 0.000 (0.000) & 0.000 (0.000) \\
        ISDEA-homo                      & 51.000 (0.000) & 0.020 (0.000) & 0.000 (0.000) & 0.000 (0.000) & 0.000 (0.000) & 0.000 (0.000) \\
        IS-DEA~\citep{kg-equivariance}  &  18.043 (1.957) & \underline{0.224} (0.006) & \underline{0.117} (0.009) & \underline{0.224} (0.012) & \underline{0.330} (0.044) & 0.473 (0.037) \\
        InGram~\citep{lee2023ingram} & \textbf{9.678} (1.749) & 0.199 (0.040) & 0.052 (0.034) & 0.179 (0.079) & 0.313 (0.081) & \textbf{0.603} (0.154) \\
        \midrule
        MTDEA ($\hat{K}=2$)             & \underline{15.001} (4.420) & \textbf{0.279} (0.062) & \textbf{0.163} (0.057) & \textbf{0.294} (0.076) & \textbf{0.387} (0.082) & \underline{0.545} (0.069) \\
        MTDEA ($\hat{K}=4$)              & 25.124 (7.784) & 0.126 (0.126) & 0.059 (0.102) & 0.101 (0.150) & 0.140 (0.169) & 0.233 (0.179) \\
        MTDEA ($\hat{K}=6$)              & 20.160 (1.464) & 0.110 (0.041)  & 0.023 (0.017) & 0.067 (0.056) & 0.128 (0.096) & 0.315 (0.127) \\
        \bottomrule
    \end{tabular}
    }

\end{table}

\subsection{More \textsc{WikiTopics-MT} Scenarios} \label{app:wikitopics-mt}

\Cref{tab:wikitopics-mt1-results,tab:wikitopics-mt2-result,tab:wikitopics-mt3-result,tab:wikitopics-mt4-result} show the additional experiment results on multi-task scenarios \textsc{WikiTopics-MT2}, \textsc{WikiTopics-MT3}, and \textsc{WikiTopics-MT4}. In most cases, our MTDEA model outperforms the baseline models on the MRR and Hits@1, while being comparable in other metrics. 

\begin{table}
    \centering
    \caption{Model performance on \textsc{WikiTopics-MT2} under the \textbf{dual-sampling metrics}, tested on two topics (\textsc{Location} and \textsc{Science}) not seen in training (\textsc{Sport} + \textsc{Health}). We report mean and std across 3 random seeds. For our MTDEA, $\hat{K}$ denotes the maximum number of tasks the architecture can model (\Cref{sec:attn-weights}).} \label{tab:wikitopics-mt2-result}
    \vspace{1pt}
    \resizebox{\linewidth}{!}{
    \begin{tabular}{lrrrrrrrrrr}
        \toprule
         & \multicolumn{4}{c}{Test on \textsc{Location} topic} & \multicolumn{4}{c}{Test on \textsc{Science} topic} \\
        \cmidrule(r){2-5} \cmidrule(l){6-9} 
        Models & {MR $\downarrow$} & {MRR $\uparrow$} & {Hits@1 $\uparrow$} & {Hits@10 $\uparrow$} & {MR $\downarrow$} & {MRR $\uparrow$} & {Hits@1 $\uparrow$} & {Hits@10 $\uparrow$} \\
        \midrule
         NBFNet-homo                    & 18.231 (1.027) & 0.109 (0.004) & 0.035 (0.002) & 0.292 (0.016) & 19.525 (0.908) & 0.091 (0.002) & 0.024 (0.001) & 0.235 (0.008)  \\
        IS-DEA-homo                     & 39.716 (1.022) & 0.027 (0.001) & 0.000 (0.000) & 0.000 (0.000)   & 37.581 (2.115) & 0.028 (0.001) & 0.000 (0.000) & 0.000 (0.000)  \\
        IS-DEA~\citep{kg-equivariance}  & \underline{14.873} (1.270)  & 0.462 (0.031) & 0.393 (0.038) & 0.569 (0.005)  & \underline{12.991} (1.067)  & 0.473 (0.024) & 0.390 (0.034) & 0.617 (0.023) \\
        InGram~\citep{lee2023ingram} & \textbf{7.234} (0.905) & 0.283 (0.047) & 0.091 (0.043) & \textbf{0.780} (0.049) & \textbf{8.447} (0.939) & 0.224 (0.069) & 0.063 (0.053) & \textbf{0.691} (0.052) \\
        \midrule
        MTDEA ($\hat{K}=2$)             & 15.619 (1.262)  & \underline{0.480} (0.178) & \underline{0.417} (0.023) & 0.557 (0.014) & 13.795 (3.410)  & \underline{0.482} (0.007) & \underline{0.406} (0.008) & \underline{0.595} (0.027) \\
        MTDEA ($\hat{K}=4$)             & 16.864 (0.704)  & 0.470 (0.042) & 0.405 (0.049) & 0.547 (0.037) & 13.302 (1.518)  & \textbf{0.483} (0.002) & \textbf{0.409} (0.004) & 0.590 (0.010) \\
        MTDEA ($\hat{K}=6$)             & 16.362 (0.933)  & \textbf{0.490} (0.001) & \textbf{0.431} (0.004) & \underline{0.558} (0.014) & 14.531 (1.121)  & 0.476 (0.007) & 0.405 (0.012) & 0.580 (0.024)  \\
        \bottomrule
    \end{tabular}
    }

    \bigskip

    \centering
    \caption{Model performance on \textsc{WikiTopics-MT3} under the \textbf{dual-sampling metrics}, tested on two topics (\textsc{Art} and \textsc{Infrastructure}) not seen in training (\textsc{People} + \textsc{Taxonomy}). We report mean and std across 3 random seeds. For our MTDEA, $\hat{K}$ denotes the maximum number of tasks the architecture can model (\Cref{sec:attn-weights}).} \label{tab:wikitopics-mt3-result}
    \vspace{1pt}
    \resizebox{\linewidth}{!}{
    \begin{tabular}{lrrrrrrrrrr}
        \toprule
         & \multicolumn{4}{c}{Test on \textsc{Art} topic} & \multicolumn{4}{c}{Test on \textsc{Infrastructure} topic} \\
        \cmidrule(r){2-5} \cmidrule(l){6-9} 
        Models & {MR $\downarrow$} & {MRR $\uparrow$} & {Hits@1 $\uparrow$} & {Hits@10 $\uparrow$} & {MR $\downarrow$} & {MRR $\uparrow$} & {Hits@1 $\uparrow$} & {Hits@10 $\uparrow$} \\
        \midrule
         NBFNet-homo                    & 19.495 (0.447) & 0.099 (0.002) & 0.030 (0.001) & 0.257 (0.006) & 15.622 (0.174) & 0.137 (0.002) & 0.056 (0.001) & 0.357 (0.004)  \\
        IS-DEA-homo                     & 37.31 (0.5739) & 0.028 (0.000) & 0.000 (0.000) & 0.000 (0.000) & 35.456 (1.634) & 0.029 (0.001) & 0.000 (0.000) & 0.000 (0.000)  \\
        IS-DEA~\citep{kg-equivariance}  & \underline{12.981} (0.753) & \underline{0.475} (0.049) & \underline{0.385} (0.062) & \textbf{0.638} (0.035) & 12.438 (0.904) & \underline{0.461} (0.011) & 0.304 (0.008) & \underline{0.696} (0.011) \\
        InGram~\citep{lee2023ingram} & \textbf{11.476} (1.041) & 0.154 (0.020) & 0.025 (0.014) & 0.527 (0.085) & \textbf{6.667} (1.988) & 0.333 (0.137) & 0.147 (0.143) & \textbf{0.789} (0.109) \\
        \midrule
        MTDEA ($\hat{K}=2$)             & 13.388 (1.928) & 0.464 (0.021) & 0.376 (0.030) & \underline{0.614} (0.041) & 11.887 (1.314) & 0.456 (0.013) & \underline{0.351} (0.010) & 0.664 (0.008) \\
        MTDEA ($\hat{K}=4$)             & 14.185 (1.181) & \textbf{0.488} (0.002) & \textbf{0.405} (0.010) & 0.598 (0.004) & \underline{10.224} (0.917) & \textbf{0.466} (0.007) & \textbf{0.353} (0.006) & 0.685 (0.012) \\
        MTDEA ($\hat{K}=6$)             & OOM  & OOM & OOM & OOM & OOM & OOM & OOM & OOM  \\
        \bottomrule
    \end{tabular}
    }

    \bigskip

    \centering
    \caption{Model performance on \textsc{WikiTopics-MT4} under the \textbf{dual-sampling metrics}, tested on two topics (\textsc{Health} and \textsc{Science}) not seen in training (\textsc{Location} + \textsc{Organization}). We report mean and std across 3 random seeds. For our MTDEA, $\hat{K}$ denotes the maximum number of tasks the architecture can model (\Cref{sec:attn-weights}).} \label{tab:wikitopics-mt4-result}
    \vspace{1pt}
    \resizebox{\linewidth}{!}{
    \begin{tabular}{lrrrrrrrrrr}
        \toprule
         & \multicolumn{4}{c}{Test on \textsc{Health} topic} & \multicolumn{4}{c}{Test on \textsc{Science} topic} \\
        \cmidrule(r){2-5} \cmidrule(l){6-9} 
        Models & {MR $\downarrow$} & {MRR $\uparrow$} & {Hits@1 $\uparrow$} & {Hits@10 $\uparrow$} & {MR $\downarrow$} & {MRR $\uparrow$} & {Hits@1 $\uparrow$} & {Hits@10 $\uparrow$} \\
        \midrule
         NBFNet-homo                    & 17.083 (0.070)  & 0.124 (0.001) & 0.046 (0.001) & 0.328 (0.002) & 19.998 (0.081) & 0.094 (0.001) & 0.026 (0.001) & 0.246 (0.002)  \\
        IS-DEA-homo                     & 32.286 (0.731)  & 0.032 (0.001) & 0.000 (0.000) & 0.000 (0.000) & 36.882 (4.891) & 0.029 (0.004) & 0.000 (0.000) & 0.000 (0.000)  \\
        IS-DEA~\citep{kg-equivariance}  & 7.199 (1.047)  & \textbf{0.681} (0.002) & \textbf{0.581} (0.003) & \underline{0.819} (0.009) & 6.696 (1.160) & \underline{0.707} (0.012) & 0.612 (0.019) & \textbf{0.847} (0.012) \\
        InGram~\citep{lee2023ingram} & 9.994 (1.236) & 0.158 (0.036) & 0.020 (0.012) & 0.605 (0.121) & 14.10 (1.794) & 0.126 (0.037) & 0.017 (0.018) & 0.392 (0.126)  \\
        \midrule
        MTDEA ($\hat{K}=2$)             & \underline{6.990} (0.686)  & \underline{0.678} (0.015) & \textbf{0.581} (0.015) & 0.803 (0.029) & \underline{6.477} (0.464) & 0.704 (0.004) & \underline{0.615} (0.005) & \underline{0.827} (0.005) \\
        MTDEA ($\hat{K}=4$)             & \textbf{6.387} (1.531)  & 0.680 (0.011) & \underline{0.575} (0.013) & \textbf{0.827} (0.019) & \textbf{5.994} (0.053) & \textbf{0.715} (0.011) & \textbf{0.634} (0.011) & \underline{0.827} (0.021) \\
        MTDEA ($\hat{K}=6$)             & OOM  & OOM & OOM & OOM & OOM & OOM & OOM & OOM  \\
        \bottomrule
    \end{tabular}
    }
\end{table}

\revision{
\subsection{General Dataset \textsc{FBNELL}}

We also experiment on the commonly used datasets \textsc{FB15K-237}~\cite{schlichtkrull2018modeling} and \textsc{NELL-995}~\cite{Nell995} by combining them and creating the \textsc{FBNELL} dataset. Specifically, the training graph consists of the 50 most frequent relation types and the test graph of the 100 most frequent ones for each dataset. We note that it is not clear from this construction whether \textsc{FBNELL} exhibits multi-task structures because the relation types might still be exchangeable despite belonging to different domains.

\Cref{tab:fbnell-result} shows the performance on \textsc{FBNELL}. Our model is on par (and sometimes outperforms) \revision{ISDEA}, even in this scenario where a multi-task structure may not be present. We associate smaller performance gaps to the simplicity of the constructed dataset, which does not seem to exhibit complex multi-task structures (the smallest $\hat{K}$ has the highest performance). 
Overall, our result suggests that in real-world scenarios might be advantageous to employ the MTDEA model because, even in the single-task setting (due to its regularization towards fewer relation equivalence classes and patterns), it obtains a similar, if not better, performance than the single-task counterpart ISDEA. \revision{Notably, MTDEA is outperformed by InGram on MR and Hit@10 metrics. Exploring the development of a multi-task version of InGram could potentially result in a model that always outperforms it, but we leave this avenue for future research}.
}

\begin{table}
    \centering
    \caption{Model performance on \textsc{FBNELL} under the \textbf{dual-sampling metrics}. We report mean and std across 3 random seeds. For our MTDEA, $\hat{K}$ denote the maximum number of tasks the architecture can model (\Cref{sec:attn-weights}). 
    } \label{tab:fbnell-result}
    \footnotesize
    \begin{tabular}{lrrrr}
        \toprule
        Models & {MR $\downarrow$} & {MRR $\uparrow$} & {Hits@1 $\uparrow$} & {Hits@10 $\uparrow$} \\
        \midrule
        NBFNet-homo                     & 14.116 (0.029)   & 0.129 (0.001) & 0.042 (0.001) & 0.379 (0.002)  \\
        ISDEA-homo                     & 31.625 (0.940)  & 0.033 (0.001) & 0.000 (0.000) & 0.000 (0.000)  \\
        ISDEA~\citep{kg-equivariance}  & 10.925 (0.383)  & \textbf{0.624} (0.010) & \textbf{0.562} (0.014) & 0.697 (0.010) \\
        InGram~\citep{lee2023ingram}  & \textbf{4.258} (0.612) & {0.456} (0.051) & 0.251 (0.052) & \textbf{0.932} (0.037)\\
        \midrule
        MTDEA ($\hat{K}=2$)             & \underline{9.106} (0.162)   & \underline{0.622} (0.012) & \underline{0.553} (0.010) & \underline{0.704} (0.024)  \\
        MTDEA ($\hat{K}=4$)             & 10.730 (0.666)  & 0.606 (0.006) & 0.543 (0.007) & 0.680 (0.023)  \\
        MTDEA ($\hat{K}=6$)             & 10.386 (0.683)  & 0.609 (0.015) & 0.547 (0.017) & 0.678 (0.008)  \\
        \bottomrule
    \end{tabular}
\end{table}

\newpage
\section{Time and Space Complexity}
\revision{
In this section we analyze the complexity of our model, focusing on \Cref{eq:asymcond-eqvlin}. We assume $L_1^{(t)}, L_2^{(t)}, L_3^{(t)}: \sR^{N \times N \times d} \to \sR^{N \times N \times d'}$ to be GNNs that output \emph{node representations} instead of pairwise representations, which is the setup we adopt in our experimental evaluation, as described in \Cref{app:exp-details}. We consider the feature dimension to be a constant.

For input graph with $N$ nodes and $R$ relation types, denote by $\Delta_{\mathrm{max}}$ the maximum node degree, and let $\hat{K}$ be the maximum number of tasks our architecture can model.
The time complexity of the $L_1^{(t)}$ and $L_2^{(t)}$ components in \Cref{eq:asymcond-eqvlin} is $\mathcal O(R N \Delta_{\mathrm{max}})$, as each of the $R$ relation types is processed using a standard GNN, which has time complexity $\mathcal O(N \Delta_{\mathrm{max}})$.
The time complexity of the $L_3^{(t)}$ component in \Cref{eq:asymcond-eqvlin} 
is $\mathcal O(R \hat{K} N \Delta_{\mathrm{max}})$, since each of the $R$ relation types iterates over the $\hat{K}$ tasks and for each of them aggregates all other relation types and processes the aggregation using a standard GNN, which has time complexity $\mathcal O(N \Delta_{\mathrm{max}})$. Therefore, our method, as described in \Cref{eq:asymcond-eqvlin}, has an overall time complexity $\mathcal O(R \hat{K} N \Delta_{\mathrm{max}})$. In practice, $\hat{K}$ is small compared to $R$ and $N$ (the maximum value we consider in our experiments is $\hat{K}=6$).

We note that the complexity of our method can be reduced if we replace the set aggregation inside $L_3^{(t)}$ to avoid excluding the current relation type. That is, if we substitute $\sum_{r'' \in \cR \setminus \{ r \}} \alpha_{r'', k} \hrepr^{(t)}_{\cdot, r'', \cdot, \cdot}$ with $\sum_{r'' \in \cR} \alpha_{r'', k} \hrepr^{(t)}_{\cdot, r'', \cdot, \cdot}$, then for each of the $\hat{K}$ tasks, the output of $L_3^{(t)}$ can be computed only once, instead of computing it for each relation type.
Therefore, the overall time complexity of our method can be improved to $\mathcal O(R N \Delta_{\mathrm{max}})$ by a simple change in the set aggregation function.

The space complexity of our method is $\mathcal O(R (N+ N \Delta_{\mathrm{max}}) + R\hat{K})$, as for each relation type we need to store $N$ 
node features and its connectivity, as well as the attention weights $\alpha \in [0, 1]^{R \times \hat{K}}$.
}

\section{Limitations}
Despite the contributions and advancements made in this work, there are aspects that can be further refined and explored in future works:
\begin{itemize}
\item \textbf{Scalability:} The proposed model may face challenges when scaling up to extremely large graphs, as memory demands might become prohibitively high. Further research is necessary to develop approximation techniques to handle such large-scale applications.

\item \textbf{Model complexity:} The proposed model introduces additional complexity compared to some baseline methods. Efforts to simplify the models while preserving their performance benefits are worth exploring in future research.

\item \textbf{Non-exchangeable relations:} There may exist cases where no relations are exchangeable, rendering it necessary to have a number of task partitions equal to the number of relations. In such situations, the benefits of our proposed method may be reduced. Investigating these cases remains an important avenue for future research.

\end{itemize}

\end{document}